\documentclass[twoside,11pt]{article}
\usepackage{simple-style}
\usepackage{amsmath}
\usepackage{amssymb}
\usepackage{amsthm}
\usepackage{mathrsfs}
\usepackage{booktabs}
\usepackage{graphicx}
\usepackage{wrapfig}
\usepackage[framemethod=TikZ]{mdframed}
\usepackage{natbib}
\usepackage[algoruled, plain, commentsnumbered]{algorithm2e}
\mdfsetup{roundcorner=10pt,linecolor=black!5,backgroundcolor=black!5,nobreak=true}

\newcommand{\field}[1]{\mathbb{#1}}

\newcommand{\set}[1]{\mathcal{#1}}
\newcommand{\nats}{\field{N}}
\newcommand{\reals}{\field{R}}
\newcommand{\comment}[1]{}

\newcommand{\drawnfrom}{\sim}       
\newcommand{\expect}{\mathsf{E}}    
\newcommand{\prob}{\mathsf{Pr}}      
\newcommand{\cost}{\mathsf{C}}      

\newtheorem{theorem}{Theorem}

\newenvironment{algoframed}{\begin{mdframed}\begin{algorithm}[H]}{\end{algorithm}\end{mdframed}}

\newenvironment{procframed}{\begin{mdframed}\begin{procedure}[H]}{\end{procedure}\end{mdframed}}

\jheading{December 2015}
\ShortHeadings{Information-Theoretic Bounded Rationality}{Ortega et al.}
\firstpageno{1}

\begin{document}

\title{Information-Theoretic Bounded Rationality}

\author{%
  \name Pedro A. Ortega \email ope@seas.upenn.edu\\
  \addr University of Pennsylvania\\
  Philadelphia, PA 19104, USA
  \AND
  \name Daniel A. Braun \email daniel.braun@tuebingen.mpg.de\\
  \addr Max Planck Institute for Intelligent Systems\\
  Max Planck Institute for Biological Cybernetics\\
  72076 T\"ubingen, Germany
  \AND
  \name Justin Dyer \email jsdyer@google.com\\
  \addr Google Inc.\\
  Mountain View, CA 94043, USA
  \AND
  \name Kee-Eung Kim \email kekim@cs.kaist.ac.kr\\
  \addr Korea Advanced Institute of Science and Technology\\
  Daejeon, Korea 305-701
  \AND
  \name Naftali Tishby \email tishby@cs.huji.ac.il\\
  \addr The Hebrew University\\
  Jerusalem, 91904, Israel
}


\maketitle

\begin{abstract}%
Bounded rationality, that is, decision-making and planning under resource limitations, is widely regarded as an important open problem in artificial intelligence, reinforcement learning, computational neuroscience and economics. This paper offers a consolidated presentation of a theory of bounded rationality based on information-theoretic ideas. We provide a conceptual justification for using the free energy functional as the objective function for characterizing bounded-rational decisions. This functional possesses three crucial properties: it controls the size of the solution space; it has Monte Carlo planners that are exact, yet bypass the need for exhaustive search; and it captures model uncertainty arising from lack of evidence or from interacting with other agents having unknown intentions. We discuss the single-step decision-making case, and show how to extend it to sequential decisions using equivalence transformations. This extension yields a very general class of decision problems that encompass classical decision rules (e.g.~\textsc{Expectimax} and \textsc{Minimax}) as limit cases, as well as trust- and risk-sensitive planning. 
\end{abstract}


\newpage

\begin{footnotesize}
\tableofcontents
\end{footnotesize}

\newpage


\section{Introduction}

It is hard to overstate the influence that the economic idea of \textit{perfect rationality} has had on our way of designing artificial agents \citep{RussellNorvig2010}. Today, many of us in the fields of artificial intelligence, control theory, and reinforcement learning, design our agents by encoding the desired behavior into an objective function that \emph{the agents must optimize in expectation}. By doing so, we are relying on the theory of \textit{subjective expected utility} (SEU), the standard economic theory of decision making under uncertainty \citep{Neumann1944, Savage1954}. SEU theory has an immense intuitive appeal, and its pervasiveness in today's mindset is reflected in many widely-spread beliefs: \textit{e.g.} that probabilities and utilities are orthogonal concepts; that two options with the same expected utility are equivalent; and that randomizing can never improve upon an optimal deterministic choice. Put simply, if we find ourselves violating SEU theory, we would feel strongly compelled to revise our choice.

Simultaneously, it is also well-understood that SEU theory prescribes policies that are intractable to calculate save for very restricted problem classes. This was recognized soon after expected utility theory was formulated \citep{Simon1956}. In agent design, it became especially apparent more recently, as we continue to struggle in tackling problems of moderate complexity in spite of our deeper understanding of the planning problem \citep{Duff2002, Hutter2004, Legg2008, Ortega2011} and the vast computing power available to us. For instance, there are efficient algorithms to calculate the optimal policy of a \emph{known} Markov decision process (MDP) \citep{Bertsekas1996}, but no efficient algorithm to calculate the exact optimal policy of an \emph{unknown} MDP or a \emph{partially observable} MDP \citep{PapadimitriouTsitsiklis1987}. Due to this, in practice we either make severe domain-specific \emph{simplifications}, as in linear-quadratic-Gaussian control problems \citep{Stengel1994}; or we \emph{approximate} the ``gold standard'' prescribed by SEU theory, exemplified by the reinforcement learning algorithms based on stochastic approximations \citep{Sutton1998, Szepesvari2010} and Monte-Carlo tree search \citep{Kocsis2006, Veness2011, Mnih2015}.

Recently, there has been a renewed interested in models of \emph{bounded rationality} \citep{Simon1972}. Rather than approximating perfect rationality, these models seek to formalize decision-making with limited resources such as the time, energy, memory, and computational effort allocated for arriving at a decision. The specific way in which this is achieved varies across these accounts. For instance, \textit{epsilon-optimality} only requires policies to be ``close enough'' to the optimum \citep{Dixon2001}; \textit{metalevel rationality} proposes optimizing a trade-off between utilities and computational costs \citep{Zilberstein2008}; \textit{bounded optimality} restricts the computational complexity of the programs implementing the optimal policy \citep{Russell1995b}; an approach that we might label \textit{procedural bounded rationality} attempts to explicitly model the limitations in the decision-making procedures \citep{Rubinstein1998}; and finally, the \textit{heuristics} approach argues that general optimality principles ought to be abandoned altogether in favor of collections of simple heuristics \citep{Gigerenzer2001}.

Here we are concerned with a particular flavor of bounded rationality, which we might call ``information-theoretic'' due to its underlying rationale. While this approach solves many of the shortcomings of perfect rationality in a simple and elegant way, it has not yet attained widespread acceptance from the mainstream community in spite of roughly a decade of research in the machine learning literature. As is the case with many emerging fields of research, this is partly due to the lack of consensus on the interpretation of the mathematical quantities involved. Nonetheless, a great deal of the basics are well-established and ready for their widespread adoption; in particular, some of the algorithmic implications are much better understood today. Our goal here is to provide a consolidated view of some of the basic ideas of the theory and to sketch the intimate connections to other fields. 

\subsection{A Short Algorithmic Illustration}

\paragraph{Perfect Rationality.}
Let $\Pi = \{\pi_1, \pi_2, \ldots, \pi_N\}$ be a finite set of $N \in \mathbb{N}$ candidate choices or policies, and let $U: \Pi \rightarrow [0,1]$ be a utility function mapping each policy into the unit interval.  Consider the problem of finding a maximizing element $\pi^\ast \in \Pi$. For simplicity we assume that, given an element $\pi \in \Pi$, its utility $U(\pi)$ can be evaluated in a constant number of computation steps. Imposing no particular structure on $\Pi$, we find $\pi^\ast$ by sequentially evaluating each utility $U(\pi)$ and returning the best found in the end.

\begin{procframed}
\caption{PerfectlyRationalChoice($\Pi, U$)}
 \DontPrintSemicolon
 $\pi^\ast \leftarrow \pi_1 \in \Pi$ \;
 \ForEach{$\pi \in \{\pi_2, \pi_3, \ldots, \pi_N\}$}{
   \lIf{$U(\pi) \geq U(\pi^\ast)$}{$\pi^\ast \leftarrow \pi$}
 }
 \Return $\pi^\ast$
\end{procframed}

This \textit{exhaustive evaluation} algorithm works well if~$N$ is small, but it does not scale to the case when~$N$ is \emph{very} large. In real-world tasks, such \emph{very large decision spaces} are not the exception but rather the norm, and an agent faces them in most stages of its information-processing pipeline (\textit{e.g.}~attention focus, model selection, action selection). In these cases, the agent does not have enough resources to exhaustively evaluate each element in $\Pi$; rather, it only manages to inspect a negligible subset of $K \ll N$ elements in $\Pi$, where $K$ depends on the agent's limitations.

\paragraph{Bounded Rationality.}
We model this limitation as a constraint on the agent's \emph{information capacity} of relating utility functions with decisions. This limitation changes the nature of the decision problem in a fundamental way; indeed, information constraints yield optimal choice algorithms that look very unusual to someone who is used to the traditional optimization paradigm. 

One simple and general way of implementing the decision process is as follows. A bounded rational agent inspects the choices until it finds one that is \textit{satisficing}. The order of inspection is random: the agent draws (without replacement) the candidate policies from a distribution $Q$ that reflects the \emph{prior knowledge} about the good choices. Then, for each proposed element $\pi \sim Q$, the agent decides whether $\pi$ is good enough using a stochastic criterion that depends both on the utility and the capacity. Importantly, the agent can inspect the policies \emph{in parallel} as is illustrated in the following pseudo-code.
\begin{procframed}
\caption{BoundedRationalChoice($\Pi,U, \alpha, U^\ast$)}
 \DontPrintSemicolon
 \SetKwBlock{llRepeat}{repeat in parallel}{end}
 \llRepeat{
   $\pi \sim Q$ \;
   $u \sim \mathcal{U}[0,1]$ \;
   \lIf{$u \leq \exp\{\alpha (U(\pi) - U^\ast)\}$}{\Return $\pi$}
 }
\end{procframed}
The stochastic criterion is parametrized by $\alpha \in \mathbb{R}$ and $U^\ast \in [0,1]$ which jointly determine the size of the solution space and thereby the choice difficulty. The parameter~$\alpha$ is known as the \textit{inverse temperature}, and it plays the role of the \emph{degree of rationality}; whereas $U^\ast$ is an \emph{aspiration level} which is typically chosen to be the maximum utility that can be attained in principle, \textit{i.e.} $U^\ast = 1$ in our example. 

The behavior of this algorithm is exemplified in Fig.~\ref{fig:intro-simulation}. In this simulation, the size $N$ of the policy space $\Pi$ was set to one million. The utilities were constructed by first mapping $N$ uniformly spaced points in the unit interval into $f(x) = x^2$ and then randomly assigning these values to the policies in $\Pi$. For the prior $Q$ we chose a uniform prior. Together, the utility and the prior determine the shape of the choice space illustrated in the left panel. The remaining panels show the performance of the algorithm in terms of the utility (center panel) and the number of inspected policies before a satisficing policy was found (right panel). This simulation shows two important aspects of the algorithm. First, the performance in utility is marginally decreasing in the search effort, as illustrated by the concavity of $U(\alpha)$ and the roughly proportional dependency between search effort and $\alpha$. Second, the parameter $\alpha$ essentially controls the performance of the algorithm. In contrast, it is easy to see that this performance is unaffected by the size of the set of policies, as long as the shape of the choice space is preserved.

\paragraph{Properties.}
Although the algorithm seems counter-intuitive at a first glance, it has the following advantages:
\begin{enumerate}
  \item \emph{Use of prior knowledge.} If the agent possesses a distribution $Q$ over $\Pi$, where $Q(\pi)$ denotes the probability of $\pi$ being the best policy, then it can exploit this knowledge. In particular, if it already knows the maximizer $\pi^\ast$, \textit{i.e.}\ if $Q$ is a degenerate distribution given by a Kronecker delta function $\delta[\pi = \pi^\ast]$, then it will only query said policy.
  \item \emph{Protection from adversarial enumeration.} If the agent were to inspect the policies using a fixed deterministic order, then it could fall prey to an adversarial shuffling of the utilities. The above randomization technique protects the agent from such an attack.
  \item \emph{Control of complexity.} The agent can control the size of the solution space by choosing an appropriate value for the $\alpha$ parameter. More precisely, $|\alpha| \approx 0$ corresponds to an easy problem in which most proposals $\pi \sim P$ will be accepted; and $|\alpha| \gg 0$ is a hard problem.
  \item \emph{Parallel.} The policies can be inspected in parallel, in contrast to perfect rationality\footnote{The running time of exhaustive search can be brought down to $O(\log N)$ through dynamic programming and parallelism, but the number of inspected policies is still $N$.}. Technically, this is possible because ``$\pi^\ast$ is satisficing'' is a statement in propositional logic. Instead, ``$\pi^\ast$ is a maximizer'', which can be rewritten as ``for every policy $\pi$, $\pi^\ast$ has larger utility than $\pi$'', is a statement in first-order logic with a universal quantifier that depends on exhaustive inspection. 
\end{enumerate}
The above basic choice algorithm is very general. Apart from modeling intractability, it can be used to represent choices under model uncertainty, risk and ambiguity sensitivity, and other players' attitudes towards the agent. In the main text, we will see how to derive the algorithm from first principles, how to model general information constraints in decision making, and how to extend the basic decision-making scheme to sequential scenarios.

\begin{figure}[t]
\centering
\includegraphics[width=\textwidth]{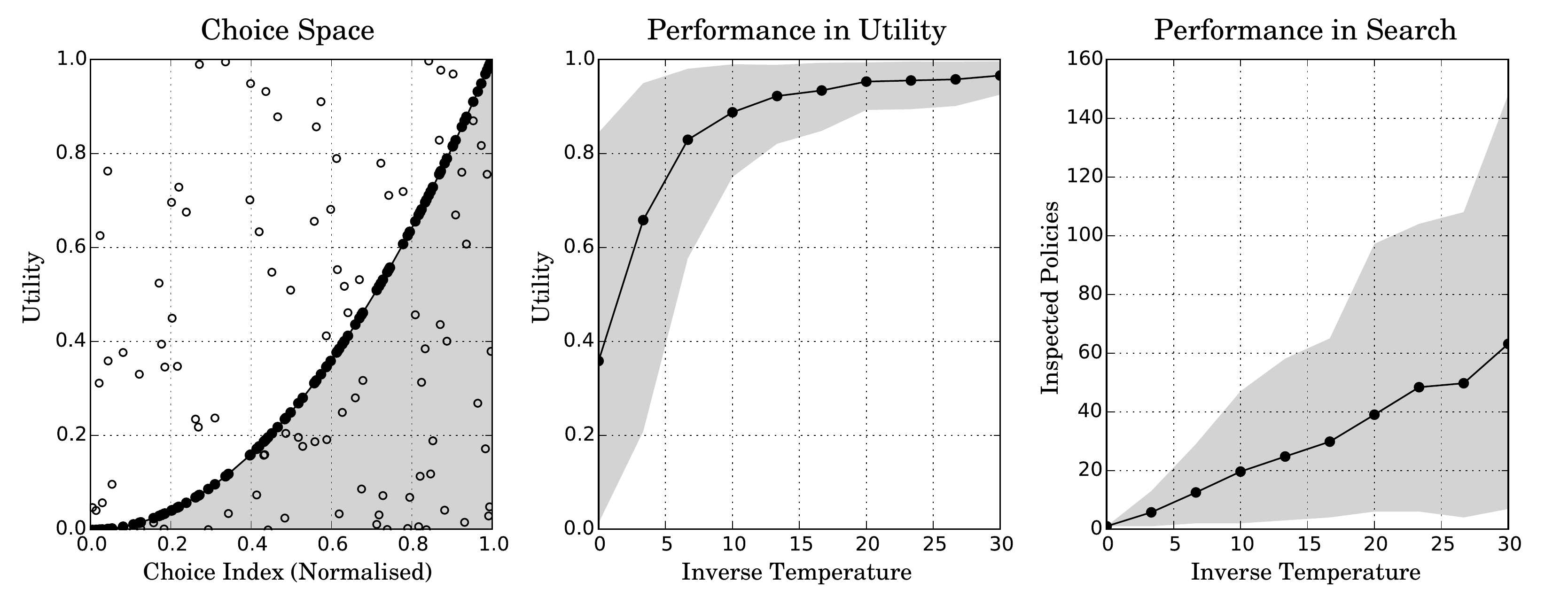}
\caption{Simulation of bounded-rational choices. The left panel shows the shape of the utility function for 100 randomly chosen policies out of one million. These are shown both in the way the agent sees them ($\circ$-markers) and sorted in ascending order of the utility ($\bullet$-markers). The shaded area under the curve is an important  factor in determining the difficulty of the policy search problem. The center and right panels illustrate the performance as a function of the inverse temperature, measured in terms of the utility (center panel) and the number of inspected policies before acceptance (right panel). Both panels show the mean (black curve) and the bands capturing 80\% of the choices excluding the first and last decile (shaded area). The values were calculated from 300 samples per setting of the inverse temperature.}\label{fig:intro-simulation}
\end{figure}

\subsection{Outlook}

In the next section we briefly review the ideas behind SEU theory. We touch upon two-player games in order to motivate a central concept: the certainty-equivalent. The section then finishes with two open problems of SEU theory: intractability and model uncertainty. Section~\ref{sec:boundedness} lays out the basic conceptual framework of information-theoretic bounded rationality. We analyze the operational characterization of having limited resources for deliberation, and state the main mathematical assumptions. From these, we then derive the free energy functional as the replacement for the expected utility in decision making. Section~\ref{sec:single-step} applies the free energy functional to single-step decision problems. It serves the purpose of further sharpening the intuition behind the free energy functional, both from a theoretical and practical view. In particular, the notion of equivalence of bounded-rational decision problems is introduced, along with practical rejection sampling algorithms. Section~\ref{sec:sequential-decisions} extends the free energy functional to sequential decisions. We will see that the notion of equivalence plays a crucial role in the construction and interpretation of bounded rational decision trees. We show how such decision trees subsume many special decision rules that represent risk-sensitivity and model uncertainty. The section finishes with a recursive rejection sampling algorithm for solving bounded-rational decision trees. The last section discusses the relation to the literature and concludes.

\section{Preliminaries in Expected Utility Theory}\label{sec:preliminaries}

Before we introduce the ideas that underlie bounded-rational decision-making, in this section we will set the stage by briefly reviewing SEU theory. 

\paragraph{Notation.}
Sets are denoted with calligraphic letters such as in $\mathcal{X}$. The set $\Delta(\mathcal{X})$ corresponds to the simplex over $\mathcal{X}$, \textit{i.e.} the set of all probability distributions over $\mathcal{X}$. The symbols $\prob$ and $\expect$ stand for the probability measure and expectation respectively, relative to some probability space with sample space $\Omega$. Conditional probabilities are defined as 
\begin{equation}\label{eq:cond-prob}
  \prob(A|B) := \frac{ \prob(A \cap B) }{ \prob(B) } 
\end{equation}
where $A,B \subset \Omega$ are measurable subsets, and where $B$ is such that $\prob(B) \neq 0$. 

\subsection{Variational principles}

The behavior of an agent is typically described in one of two ways: either \emph{dynamically}, by directly specifying the agent's actions under any contingency; or \emph{teleologically}, by specifying a \emph{variational problem} (\textit{e.g.} a convex objective function) that has the agent's policy as its optimal solution. While these two descriptive methods can be used to represent the same policy, the teleological description has a greater explanatory power because in addition it encodes a \emph{preference relation} which justifies why the policy was preferred over the alternatives\footnote{Notice also that in this sense, the qualifier ``optimal'' is a statement that only holds relative to the objective function. That is, ``being optimal'' is by no means an absolute statement, because given \emph{any policy}, one can always engineer a variational principle that is extremized by it.}.

Because of the explanatory power of variational principles, they are widely used to justify design choices. For example, virtually every sequential decision-making algorithm is conceived as a \textit{maximum expected utility} problem. This encompasses popular problem classes such as multi-armed bandits, Markov decision processes (MDPs) and partially observable Markov decision processes (POMDPs)---see \citet{RussellNorvig2010} for a discussion. In learning theory, \textit{regret minimization} \citep{Loomes1982} is another popular variational principle to choose between learning algorithms \citep{CesaBianchiLugosi2006}.

\subsection{Subjective expected utility}\label{sec:seu}

Today, the theory of \emph{subjective expected utility} \citep{Savage1954} is the \textit{de facto} theory of rationality in artificial intelligence and reinforcement learning \citep{RussellNorvig2010}. The bedrock of the theory is a representation theorem stating that the preferences of a rational agent can be described in terms of comparisons between expected utilities. The qualifiers ``subjective'' and ``expected'' in its name derive from the fact that both the utility function and the belief distribution are assumed to be properties that are specific to each agent, and that the utility of a random realization is equated with the expected utility over the individual realizations.

\paragraph{Decision problem.}
Let $\set{X}$ and $\set{Y}$ be two finite sets, the former corresponding to the \emph{set of actions} available to the agent and the latter to the \emph{set of observations} generated by the environment in response to an action. A \emph{realization} is a pair $(x,y) \in \set{X} \times \set{Y}$. Furthermore, let $U: (\set{X} \times \set{Y}) \rightarrow \reals$ be a \emph{utility function}, such that $U(x,y)$ represents the desirability of the realization $(x,y) \in (\set{X}\times\set{Y})$; and let $Q(\cdot|\cdot)$ be a conditional probability distribution where $Q(y|x)$ represents the probability of the observation $y \in \set{Y}$ given the action $x \in \set{X}$. 

\paragraph{Optimal policy.} The agent's goal is to choose a \emph{policy} that yields the highest expected utility; in other words, a probability distribution $P \in \Delta(\set{X})$ over actions that maximizes the functional (see Fig.~\ref{fig:eu-simplex})
\begin{equation}\label{eq:expected_utility}
  EU[\tilde{P}]
  := \sum_x \tilde{P}(x) \expect[U|x]
  = \sum_x \tilde{P}(x) \biggl\{ \sum_y Q(y|x) U(x,y) \biggl\}.
\end{equation}
Thus, an \emph{optimal policy} is any distribution $P$ with no support over suboptimal actions, that is, $P(x) = 0$ whenever $\expect[U|x] \leq \max_z \expect[U|z]$. In particular, because the expected utility is linear in the policy's probabilities, one can always choose a solution that is a vertex of the probability simplex $\Delta(\set{X})$:
\begin{equation}\label{eq:optimal_expected_utility}
  P(x)
  = \delta_{x^\ast}^x
  = \begin{cases}
    1 & \text{if $x = x^\ast := \arg\max_x \expect[U|x]$}\\
    0 & \text{otherwise,}
    \end{cases}
\end{equation}
where $\delta$ is the Kronecker delta. Hence, there always exists a deterministic optimal policy. 

\begin{figure}
\centering
\includegraphics[width=0.6\textwidth]{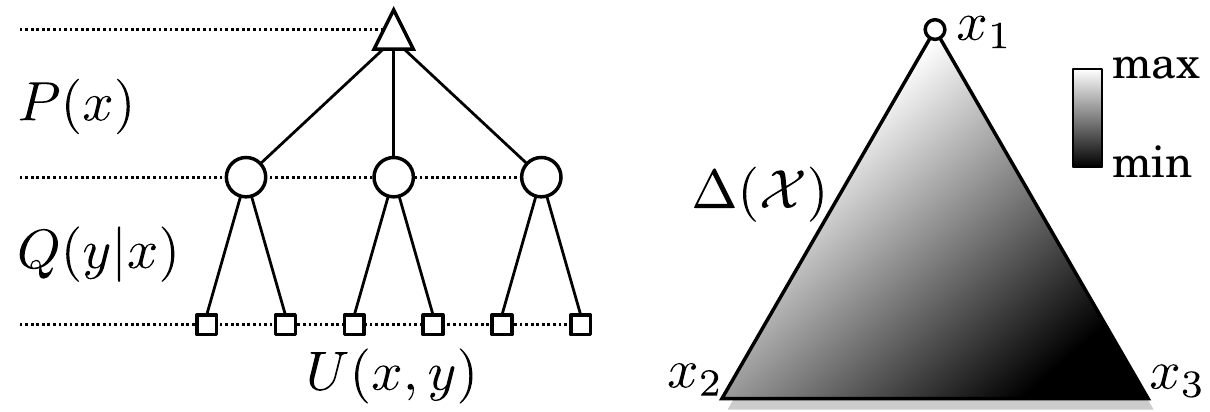}
\caption{Expected Utility Theory. Left: A decision problem can be represented as a tree with two levels. Choosing a policy amounts to assigning transition probabilities for the decision node located at the root~($\bigtriangleup$), subject to the fixed transition probabilities in the chance nodes~($\bigcirc$) and the utilities at the leaves ($\square$). Right: Mathematically, this is equivalent to choosing a member $P$ of the simplex over actions $\Delta(\mathcal{X})$ that maximizes the convex combination over conditional expected utilities $\expect[U|x]$.}\label{fig:eu-simplex}
\end{figure}

\paragraph{Utility distribution.} Given the optimal policy, the utility $U$ becomes a well-defined random variable with probability distribution
\begin{equation}\label{eq:utility_distribution}
  \prob(U=u)
  = \prob\bigl\{ (x,y) : U(x,y) = u \bigr\} \nonumber \\
  = \sum_{U^{-1}(u)} P(x) Q(y|x). 
\end{equation}
Even though the optimal policy might be deterministic, the utility of the ensuing realization is in general a non-degenerate random variable. Notably, perfectly rational agents are insensitive to the higher-order moments of the utility: so two actions yielding utilities having very different variances are regarded as being equal if their expectations are the same.

\paragraph{Friends and foes.}
What happens when the environment is yet another agent? According to \emph{game theory} \citep{Neumann1944}, we can model this situation again as a decision problem, but with the crucial difference that the agent lacks the conditional probability distribution $Q(\cdot|\cdot)$ over observations that it needs in order evaluate expected utilities. Instead, the agent possesses a second utility function $V: (\set{X} \times \set{Y}) \rightarrow \reals$ representing the desires of the environment. Game theory then invokes a \emph{solution concept}, most notably the \emph{Nash equilibrium}, in order to propose a substitute for the missing distribution $Q(\cdot|\cdot)$ based on $U$ and $V$, thereby transforming the original problem into as a well-defined decision problem.

To simplify, we assume that the agent chooses first and the environment second, and we restrict ourselves to two special cases: (a) the fully \emph{adversarial} case $U(x,y) = -V(x,y)$; and (b) the fully \emph{friendly} case $U(x,y) = V(x,y)$. The Nash equilibrium then yields the decision rules \citep{Osborne1999}
\begin{align}
  P &= \arg\max_{\tilde{P}} \sum_x \tilde{P}(x) 
    \biggl\{ \min_{\tilde{Q}} \tilde{Q}(y|x) U(x,y) \biggr\},
  \label{eq:minmax} \\
  P &= \arg\max_{\tilde{P}} \sum_x \tilde{P}(x) 
    \biggl\{ \max_{\tilde{Q}} \tilde{Q}(y|x) U(x,y) \biggr\}
  \label{eq:maxmax}
\end{align}
for the two cases respectively. Comparing these to~\eqref{eq:expected_utility}, we observe two properties. First, the overall objective function appears to arise from a modular composition of two nested choices: the outer for the agent and the inner for the environment. Second, there are essentially three ways in which the utilities of a choice are aggregated---namely through maximization, expectation, and minimization---depending on the nature of the choice. 

\paragraph{Certainty-equivalent.} 
An agent plans ahead by recursively aggregating future value into present value. As we have seen, it does so by summarizing the value of each choice using one of three aggregation types: \emph{minimization} ($\bigtriangledown$), representing an adversarial choice; \emph{expectation} ($\bigcirc$), representing an indifferent (random) choice; and \emph{maximization} ($\bigtriangleup$), representing a friendly choice. This summarized value is known as the \emph{certainty-equivalent}, because if the agent were to substitute a choice with multiple alternatives with an equivalent (degenerate) choice with a single alternative, the latter would have this value (see Fig.~\ref{fig:ce-examples}). Notice that for planning purposes, it is irrelevant whether the decision is made by the agent or not; to the agent, what matters is the degree to which an outcome (be it an action or an observation) contributes positively to its objective function (as encoded by one of the three possible aggregation operators). Indeed, another (concise) way to describe subjective expected utility theory is that the certainty-equivalent of random outcomes is given by the expectation.

\begin{figure}[ht]
\centering
\includegraphics[width=0.8\textwidth]{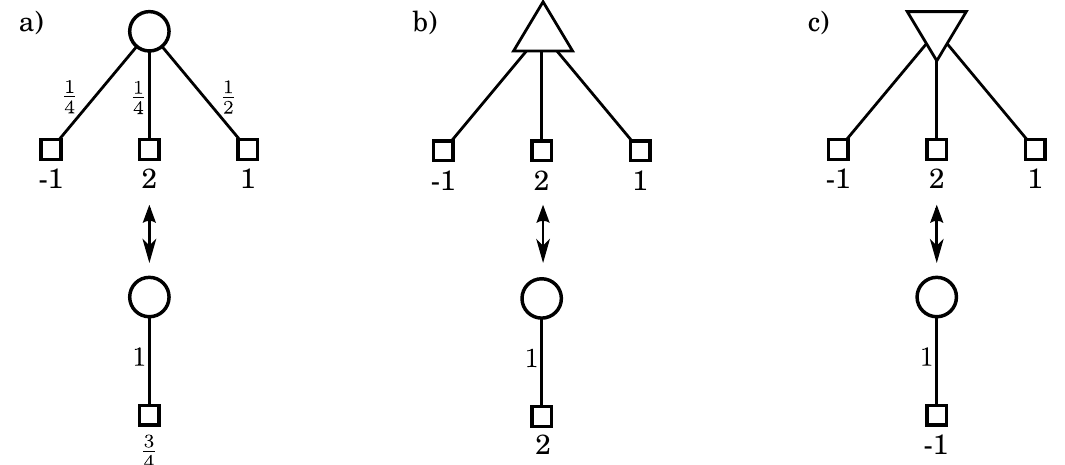}
\caption{Decision problems (top row) and their certainty-equivalents (bottom row).}\label{fig:ce-examples}
\end{figure}

\subsection{Two open problems}\label{sec:open-problems}

We finish this section by presenting two very common situations in which the application of perfect rationality does not seem to be the right thing to do: very large choice spaces and model uncertainty. 

\paragraph{Very large choice spaces.}
Consider the problem of choosing the longest straw in a given collection. Expected utility theory works well when the number of elements is small enough so that the agent can exhaustively measure all their lengths and report the longest found. This is the case in the situation depicted in Fig.~\ref{fig:large-scale}a, where it is easily seen that the top-most straw is the longest. The situation of Fig.~\ref{fig:large-scale}b however appears to be very different. Here, the number of straws is very large and their arrangement is unstructured. The lack of structure prevents us from grouping the straws for simplifying our search, \textit{i.e.} there are no symmetries for defining equivalence classes that would reduce the cardinality of our search space. Consequently, there is no obvious improvement over testing some elements by picking at random until finding one that is sufficiently good. 

\begin{figure}[h]
\centering
\includegraphics[width=\textwidth]{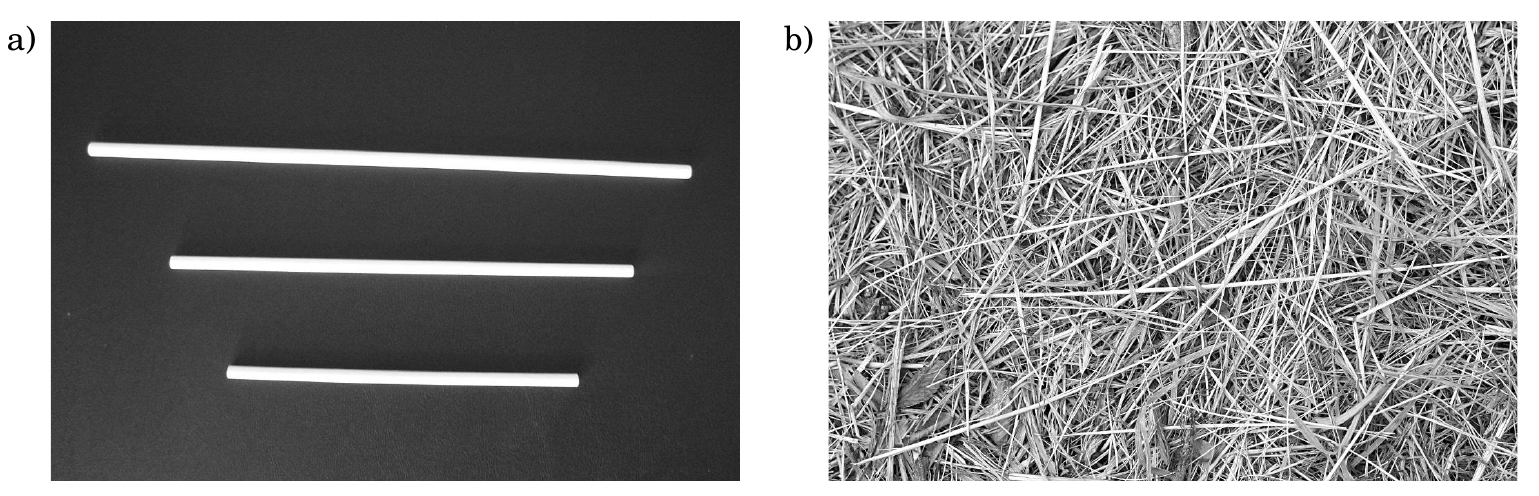}
\caption{Very large choice spaces. a) Perfect rationality can capture decision problems involving a small number of choices. b) However, perfect rationality is intractable in very large and unstructured choice spaces.}\label{fig:large-scale}
\end{figure}

\paragraph{Model uncertainty.}
Consider a task in which you must choose one of two boxes containing $100$ black and white balls. After your choice is announced, a ball is drawn at random from the box, and you win $\$ 100$ if the color is black  (or nothing otherwise). If you get to see the contents of the two boxes then you would choose the box containing a larger proportion of black balls. For instance, in Fig.~\ref{fig:ellsberg}a you would choose the left box. However, which box would you choose in the situation shown in Fig.~\ref{fig:ellsberg}b? In his seminal paper, \citet{Ellsberg1961} showed that most of people bet on the left box. A simple explanation for this phenomenon is that people fill the missing information of the right box by enumerating all the possible combinations of 100 black and white balls, attaching prior probabilities to them in a risk-averse way. Thus, such prior beliefs assigns a larger marginal probability to the loosing color---in this case, ``white''.

Now, assume that you place your bet, but no ball is drawn. Instead, you are asked to revise your bet. The contents of the boxes are assured to stay exactly the same, but now you are told that this time you win for ``white'' rather than ``black''. Would you change your bet? Empirically, most of the people stick with their initial bet. However, there is no single prior distribution over combinations that can predict this particular preference pattern, as this would contradict expected utility theory. Are people thus irrational? Perhaps, but it is worthy to point out that L. J. Savage, who formulated subjective expected utility theory, was among the people who violated the theory \citep{Ellsberg1961}. 

\begin{figure}[h]
\centering
\includegraphics[width=\textwidth]{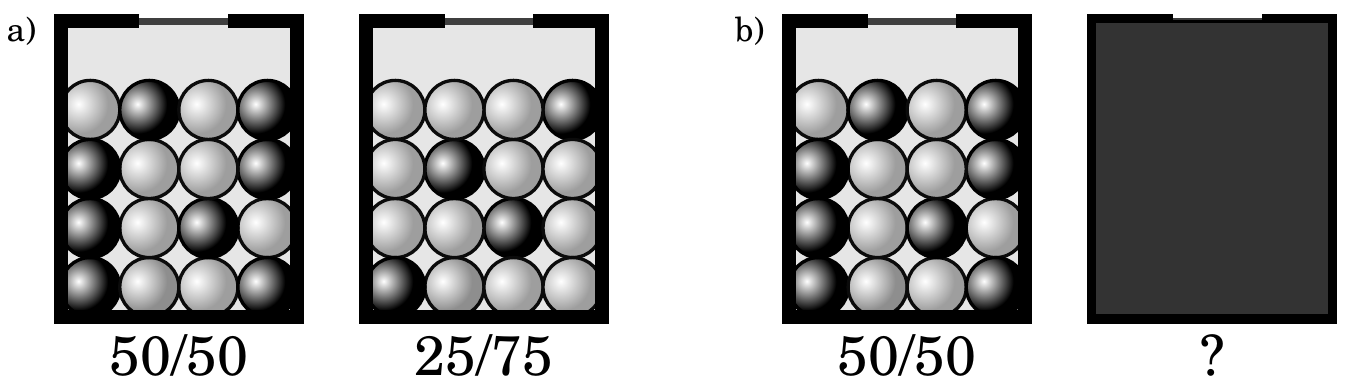}
\caption{The Ellsberg paradox. a) Two boxes with known proportions of black and white balls. b) Two boxes where only the proportions of the left box are known.}\label{fig:ellsberg}
\end{figure}

One explanation for this apparent paradox is that the right box of second experiment involves unknown probabilities\footnote{Importantly, the economic literature distinguishes between \textit{risk} (= known probabilities) and \textit{ambiguities} (= unknown probabilities) \citep{Knight1921, Ellsberg1961}. Savage defined subjective expected utility theory as a theory of decision making \emph{under risk}. There is currently no widely agreed upon definition of ambiguities, although there exist proposals, such as the framework of \textit{variational preferences} \citep{Rustichini2006}.} which affect the net value of a choice. In this case, people might actually have a uniform model over combinations of black and white balls, but they nevertheless discount the value of the box as an expression of \textit{distrust} in the model. On the other hand, there are people who perceive increased net values in some situations. Such could be the case, for instance, if they knew and trusted the person who set up the experiment. We will see later in the text how this interaction between the trust in the beliefs and the utility can be modeled using the bounded-rational framework.

\section{The Mathematical Structure of Boundedness}\label{sec:boundedness}

We now tackle the problem of formalizing boundedness. First, we will pursue an obvious route, namely that of \textit{meta-reasoning}. While meta-reasoning models can model important aspects of complex decision-making (\textit{e.g.} preliminary resource allocation and reasoning about other agents), it ultimately falls short in adequately modeling bounded-rationality.

\subsection{Meta-reasoning}\label{sec:meta-resoning}

One straightforward attempt to address the problem of intractability of expected utility theory is through \emph{meta-reasoning} \citep{Zilberstein2008}, \textit{i.e.}\ by letting the agent reason about the very costs of choosing a policy. The rationale is that, an agent can avoid prohibitive costs by finding a trade-off between the utility of a policy and the cost of evaluating said policy. 

\paragraph{Formal solution.} The argument proceeds as follows. Let $U: \Pi \rightarrow \mathbb{R}$ be a bounded utility function that maps each policy $\pi \in \Pi$ into a value $U(\pi)$ in the unit interval. A perfectly-rational agent then solves the problem
\begin{equation}\label{eq:maximisation}
  \max_{\pi \in \Pi} U(\pi),
\end{equation}
and uses any solution $\pi^\ast$ as its policy. Rather than solving~\eqref{eq:maximisation} which is deemed to be too hard, a \emph{meta-level rational} agent solves the problem
\begin{equation}\label{eq:meta-level}
  \max_{\pi \in \Pi} \Bigl\{ U(\pi) - C(\pi) \Bigr\},
\end{equation}
where $C(\pi) \in \mathbb{R}^{+}$ is a positive penalization term due to the cost of evaluating the policy~$\pi$ that is spent by the agent's policy-search algorithm, \textit{e.g.} time or space complexity of a Turing machine. 

\paragraph{Criticism.} While the idea of meta-reasoning is intuitively appealing, it fails to simplify the original decision problem. This is seen by defining $U'(\pi) := U(\pi) - C(\pi)$ and noting that the agent's meta-level reasoning maximizes the objective function
\begin{equation}\label{eq:meta-2}
  \max_{\pi \in \Pi} U'(\pi),
\end{equation}
which is itself another perfectly-rational decision problem. But then, the rationale of meta-reasoning tells us that evaluating the meta-level utilities $U'(\pi)$ should come with penalizations $C'(\pi)$ due to the costs incurred by the policy-search algorithm that is used at the meta-level\ldots

There are two problems with this scheme. First, it is quickly seen that this line of reasoning leads to an infinite regress when carried to its logical conclusion. Every time the agent instantiates a meta-level to reason about lower-level costs, it generates new costs at the meta-level. Second, the problem at a meta-level is typically harder than the lower-level one. How can we circumvent these problems?

\subsection{Incomplete Information and Interrupted Deliberation}\label{sec:incomplete-information}

The central problem with meta-reasoning lies in its self-referential logic: whenever an agent attempts to reason about its own costs of reasoning, it creates a new perfectly-rational decision problem at the meta-level that comes with additional costs that are left out. Due to this, we conclude that it is \emph{impossible} for an agent to fully apprehend its own resources of deliberation. Notice that this does not prevent agents from using meta-reasoning; it just excludes meta-reasoning as a formal solution to the bounded-rationality problem. 

\paragraph{Formal solution.} The problem can be solved by modeling the agent's deliberation as a \textit{game of incomplete information} or \textit{Bayesian game} \citep{Harsanyi1967}. Such a game allows us to represent the agent's ignorance about the very objective function it is supposed to optimize. Loosely speaking, it allows us to model the ``unknown uncertainties'' (as opposed to ``known uncertainties'') that the agent has about its goals. We refer the reader to the texts on game theory by \citet{Osborne1999} or \citet{Leyton2008} for an introduction to Bayesian games.

We model the agent's decision as single-player game of incomplete information. For this, it is useful to distinguish between the agent's \emph{alleged} objective and the \emph{perceived} objective as seen by an external observer. The agent is equipped with an objective function $U(\pi)$ that it \emph{attempts} to optimize with respect to the policy $\pi \in \Pi$. However, it unexpectedly runs out of resources, effectively being interrupted at an indeterminate point in its deliberation, and forced to commit to a suboptimal policy $\pi^{\circ} \in \Pi$. From an external observer's point of view, $\pi^{\circ} \in \Pi$ \emph{appears} to be the result of an intentional deliberation seeking to optimize an alternative objective function $U(\pi) - C(\pi)$ that trades off utilities and deliberation costs. 

Formally, let $\mathcal{C}$ be a discrete set of penalization functions of the form $C:\Pi \rightarrow \reals$ that model different interruptions. Then, the agent does not solve a particular problem, but rather a \emph{collection of problems} given by
\begin{equation}\label{eq:interrupted}
  \forall C \in \mathcal{C},\qquad \max_{\pi \in \Pi} \Bigl\{ U(\pi) - C(\pi) \Bigr\}.
\end{equation}
In other words, the agent finds an optimal policy $\pi^\ast_C$ for \emph{any} penalization function $C \in \mathcal{C}$. 

This abstract multi-valued optimization might appear unusual. However, one simple way of thinking about it is in terms of an \textit{any-time} optimization. In this scheme, the agent's computation generates ``intermediate'' solutions $\pi_1, \pi_2, \ldots$ to
\[
  U - C_1, \quad 
  U - C_2, \quad 
  U - C_3, \quad 
  \ldots
\]
respectively, where $C_1, C_2, \ldots$ is an exhaustive enumeration of $\mathcal{C}$. This ordering of $\mathcal{C}$ is not specified by~\eqref{eq:interrupted}; however, one can link this to computational resources by demanding an ascending ordering $C_1(\pi_1) \leq C_2(\pi_2) \leq \ldots$ of penalization. Nature secretly chooses a penalization function~$C^\circ \in \mathcal{C}$, which is only revealed when the agent produces the solution~$\pi^\circ$ to $U-C^\circ$. This causes the computation to stop and the agent to return $\pi^\circ$. Now, notice that since the agent does not know $C^\circ$, its computational process can be regarded as an \textit{a priori} specification of the solution to the multi-valued optimization problem \eqref{eq:interrupted}.

\subsection{Commensurability of utility and information}\label{sec:commensurable}

The fact that finding a good policy comes at a cost that changes the net value of the agent's pay-off suggests that utilities and search costs should be translatable into each other, \textit{i.e.} they should be \emph{commensurable}. The aim of this section is to present a precise relationship between the two quantities. This relationship, while straightforward, has rich and conceptually non-trivial consequences that significantly challenge our familiar understanding of decision-making.

\paragraph{Search.} We first introduce a search model that will serve as a concrete example for our exposition. In this model, the agent searches by repeatedly obtaining a random point from a \emph{very large and unstructured domain} until a desired target is hit. For this, we consider a finite \emph{sample space}~$\Omega$ and a probability measure $\prob$ that assigns a probability $\prob(S)$ to every subset $S \subset \Omega$; a \emph{reference set} $B \subset \Omega$ with $\prob(B) \neq 0$; and a \emph{target set} $A \subset B \subset \Omega$. In each round, the agent samples a point $\omega$ from the conditional probability distribution $\prob(\cdot|B)$. If $\omega$ falls within the target set $A$, then the agent stops. Otherwise, the next round begins and the sampling procedure is repeated. 

The components of this model are interpreted as follows. The agent's \emph{prior knowledge} is modeled using the reference set~$B$. The points that lie outside of~$B$ are those that are never inspected by the agent (\textit{e.g.} because it knows that these points are undesirable or because they are simply inaccessible). Every choice of a reference set induces a \emph{prior} distribution~$\prob(\cdot|B)$. Furthermore, the difficulty of finding the target is relative: The number of samples~$N$ that the agent has to inspect up until hitting the target depends upon the relative size of~$A$ with respect to~$B$, which is given by the conditional probability $\prob(A|B)$. Thereby, $N$ is a random variable that follows a geometric distribution $N \sim \mathcal{G}(p)$ with success probability $p := \prob(A|B)$, having the expected value
\begin{equation}\label{eq:expected-target}
    \expect[ N ] = \frac{ 1 }{ \prob(A|B) }.
\end{equation}

\paragraph{Decision complexity.} Based on the previous search model, we now define the \textit{decision complexity} $\cost(A|B)$ as a measure of the cost required to specify a target set $A$ given a reference $B$, where $A$ and $B$ are arbitrary measurable subsets of the sample space $\Omega$ with $\prob(B) \neq 0$. We impose the following properties on the measure $\cost$:

\begin{mdframed}
\begin{enumerate}
\item[a)] \emph{Functional form:} For every $A, B$ such that $A \subset B \subset \Omega$,
  \[\cost(A|B) := f \circ \prob(A|B),\] 
  where $f$ is a continuous, real-valued mapping.
\item[b)] \emph{Additivity:} For every $A, B, C$ such that $A \subset B \subset C \subset \Omega$, 
  \[ \cost(A|C) = \cost(B|C) + \cost(A|B). \]
\item[c)] \emph{Monotonicity:} For every $A, B, C, D$ such that $A \subset B \subset \Omega$ and $C \subset D  \subset \Omega$,
  \[ \prob(A|B) > \prob(C|D) \quad \Longleftrightarrow \quad \cost(A|B) < \cost(C|D). \]
\end{enumerate}
\end{mdframed}

\begin{figure}[htbp]
\centering
\includegraphics[width=\textwidth]{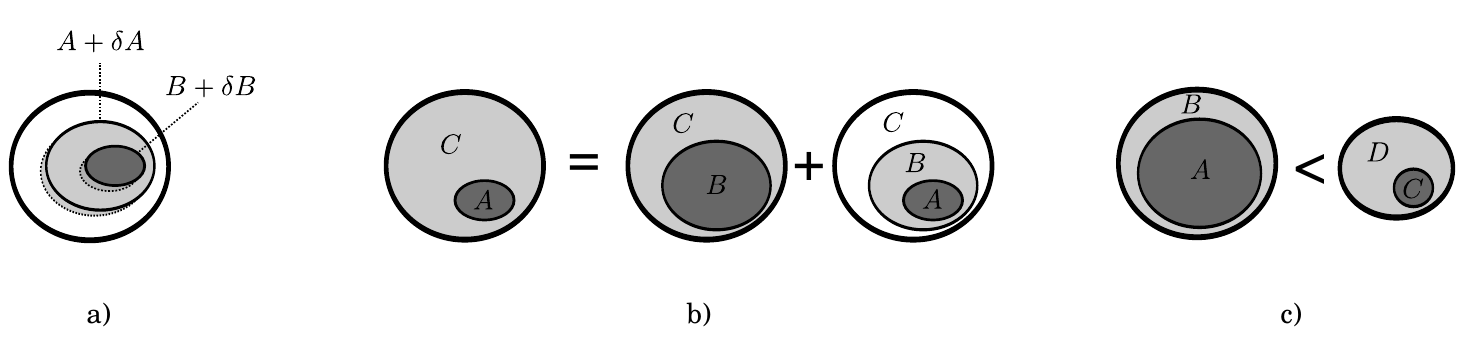}
\caption{Decision complexity. In each figure, the reference and target sets are depicted by light and dark-colored areas respectively. a) The complexity varies continuously with the target $A$ and the reference $B$. b) The complexity can be decomposed additively. c) The complexity decreases monotonically when increasing the conditional probability.}\label{fig:desiderata}
\end{figure}

How does $\cost$ look like? We first observe that the desideratum (a) allows us to restrict our attention only to continuous functions~$f$ that map the unit interval $[0,1]$ into real numbers. Then, the desiderata~(b) and~(c) imply that any~$f$ must fulfill the functional equation
\begin{equation}\label{eq:functional-equation-f}
 f(pq) = f(p) + f(q)
\end{equation}
for any $p,q \in [0,1]$ subject to the constraint $f(p) < f(q)$ whenever $p > q$. It is well-known that any solution to the previous functional equation must be of the form
\begin{equation}\label{eq:solution-f}
  f(p) = - \frac{1}{\alpha} \log p,
\end{equation}
where $\alpha > 0$ is an arbitrary positive number. Thus, the complexity measure is given by
\begin{equation}\label{eq:cost}
  \cost(A|B) = -\frac{1}{\alpha} \log P(A|B) = \frac{1}{\alpha} \log \expect[N],
\end{equation}
that is, proportional to the \emph{Shannon information} $-\log \prob(A|B)$ of $A$ given $B$. This shows that the decision complexity is proportional to the minimal amount of bits necessary to specify a choice and proportional to the logarithm of the expected number of points $\expect[N]$ that the agent has to sample before finding the target.

\paragraph{Utility.} 
Let us take one step back to revisit the concept of utility. Utilities were originally envisioned as \emph{auxiliary} quantities for conceptualizing \emph{preferences} in terms of simple numerical comparisons. In turn, a \emph{preference} is an empirical tendency for choosing one course of action over another that the agent repeats under similar circumstances. In the economic literature, this rationale that links probability of choice and preferences is known as \emph{revealed preferences} \citep{Samuelson1938}.

It is interesting to note that subjective expected utility is a theory that is particularly stringent in its behavioral demands, because it assumes that an agent will unequivocally choose the exact same course of action whenever it faces an equivalent situation---rather than just displaying a \emph{tendency} towards repeating a choice. Intuitively though, it seems natural to admit weaker forms of preferences. For instance, consider the situation in Fig.~\ref{fig:embedding}a depicting an agent's choice probabilities $Q(x)$ and $P(x)$ before and after deliberation respectively. Even though the agent does not commit to any particular choice, it is plausible that its preference relation $\succ$ is such that $x_2 \succ x_3 \succ x_1$, meaning that $U(x_2) > U(x_3) > U(x_1)$ for an appropriately defined notion of utility.

\begin{figure}
\centering
\includegraphics[width=\textwidth]{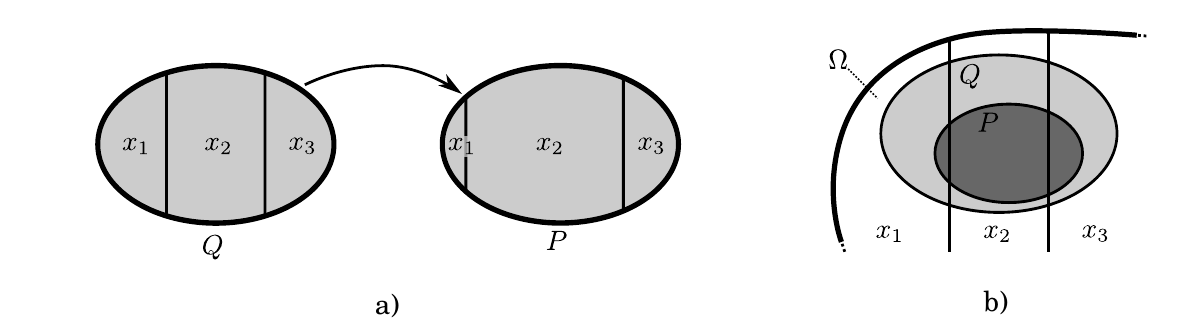} 
\caption{Deliberation as search. a) An agent's deliberation is a transformation of prior choice probabilities $Q(x)$ into posterior choice probabilities $P(x)$. The relative changes reflect the agent's preferences. b) To understand the complexity of this transformation, we model it as a search process with the convention $Q(x) = \prob(x|Q)$ and $P(x) = \prob(x|P)$, where $P \subset Q \subset \Omega$.}\label{fig:embedding}
\end{figure}

\paragraph{Free energy.}
We have discussed several concepts such as decisions, choice probabilities, utilities, and so forth. Our next step consists in reducing these to a single primitive concept: decision complexity. This synthesis is desirable both due the theoretical parsimony as well as the integrated picture of decision-making that it delivers. To proceed, we first note that an agent's deliberation process transforming a prior~$Q$ into a posterior~$P$ can be cast in terms of a search process where~$Q$ and~$P$ are the reference and target sets respectively. As illustrated in Fig.~\ref{fig:embedding}b, the set of available choices $\mathcal{X}$ forms a partition of the sample space $\Omega$. The distributions~$P$ and~$Q$ are encoded as nested subsets $P \subset Q \subset \Omega$ using the notational convention
\begin{equation}\label{eq:convention-prob}
  P(x) := \prob(x|P) \qquad \text{ and } \qquad Q(x) := \prob(x|Q),
\end{equation}
that is, where the two sets induce distributions over the choices through conditioning. With these definitions in place, we can now derive an expression for the complexity $\cost(P|Q)$ in terms of choice-specific complexities $\cost(x \cap P|x \cap Q)$:
\begin{align}
  \nonumber
  \cost(P|Q) &= -\frac{1}{\alpha} \log \prob(P|Q) \vphantom{\biggl|}\\
  \nonumber
       &= - \frac{1}{\alpha} \sum_x \prob(x|P) 
  \nonumber
         \log \biggl\{ \prob(P|Q) \frac{ \prob(x|P) \prob(x|Q) }{ \prob(x|P) \prob(x|Q) } \biggr\}\\
  \nonumber
       &= - \frac{1}{\alpha} \sum_x \prob(x|P) \log \frac{ \prob(x \cap P|Q) }{ \prob(x|Q) }
         + \frac{1}{\alpha} \sum_x \prob(x|P) \log \frac{ \prob(x|P) }{ \prob(x|Q) } \\
  \label{eq:CPQ}
       &= \sum_x \prob(x|P) \cost(x \cap P|x \cap Q)
                + \frac{1}{\alpha} \sum_x \prob(x|P) \log \frac{ \prob(x|P) }{ \prob(x|Q) }.
\end{align}
The first equality is obtained through an application of the definition~\eqref{eq:cost}. The second equality is obtained by multiplying the constant term $\prob(P|Q)$ with another term that equals one and subsequently taking the expectation with respect to $P(x|P)$. Separating the terms in the logarithm into two different sums and using the product rule $\prob(x \cap P|Q) = \prob(x|P) \cdot \prob(P|Q)$ gives the third equality (note that $P \cap Q = P$). The last step of the derivation is another application of \eqref{eq:cost}. If we now use the convention \eqref{eq:convention-prob} we get
\begin{equation}\label{eq:free-energy-cost}
  \cost(P|Q) = 
    \sum_x P(x) \cost(x \cap P|x \cap Q) 
    + \frac{1}{\alpha} \sum_x P(x) \log \frac{ P(x) }{ Q(x) }.
\end{equation}
How can we interpret this last expression? In essence, it relates the decision complexities of two different states of knowledge. If the agent's state of knowledge is $(x \cap Q)$, then the complexity of finding the target $(x \cap P)$ is equal to $\cost(x \cap P|x \cap Q)$. This assumes that the agent knows that the resulting choice will be $x$ for sure. However, since the agent does not know the final choice, then \eqref{eq:free-energy-cost} says that the total complexity of deliberation is equal to the average choice-specific complexities \emph{plus} a penalty (given by the KL-divergence) due to not knowing of the future outcome. Put differently, the certainty-equivalent complexity of an uncertain choice \emph{is larger than the expected complexity}. This constitutes the most important deviation from subjective expected utility theory.

Equation \eqref{eq:free-energy-cost} derived above can now be turned into a variational principle by observing that the r.h.s.\ is convex in the posterior choice probabilities. Thus, define \emph{utilities} as quantities that, up to a constant $C \in \mathbb{R}$, are equal to negative complexities:
\[
  U(x) := -\cost(x \cap P|x \cap Q) + C.
\]
Then, subtracting $C$ from \eqref{eq:free-energy-cost} and taking the negative gives the functional 
\begin{equation}\label{eq:free-energy-functional}
  F[\tilde{P}] := \sum_x \tilde{P}(x) U(x) 
    - \frac{1}{\alpha} \sum_x \tilde{P}(x) \log \frac{ \tilde{P}(x) }{ Q(x) },
\end{equation}
which now is \emph{concave} in $\tilde{P}$ and, when maximized, minimizes the complexity of transforming~$Q$ into~$P$ subject to the utilities~$U(x)$. Equation \eqref{eq:free-energy-functional} is the so-called \emph{free energy functional}, or simply the \emph{free energy}\footnote{Here we adopt this terminology to relate to existing work on statistical mechanical approaches to control. To be precise however, in the statistical mechanical literature this functional corresponds to (a shifted version of) the \emph{negative free energy difference}. This is because utilities are negative energies, and because \eqref{eq:free-energy-functional} characterizes the \emph{difference} between two free energy potentials.}, and it will serve as the foundation for the approach to bounded-rational decision making of this work.

\section{Single-Step Decisions}\label{sec:single-step}

\subsection{Bounded-rational decisions}

In this section we will take the free energy functional~\eqref{eq:free-energy-functional} as the objective function for modeling bounded-rational decision making. We will first focus on simple one-step decisions and explore their conceptual and algorithmic implications.

\paragraph{Decision problem.}
A bounded-rational agent, when deliberating, transforms prior choice probabilities into posterior choice probabilities in order to maximize the expected utility---but it does so subject to information constraints. Formally, a \emph{bounded-rational decision problem} is a tuple $(\alpha, \mathcal{X}, Q, U)$, where: $\alpha \in \mathbb{R}$ is the \emph{inverse temperature} which acts as a rationality parameter; $\mathcal{X}$ is a finite set of possible \emph{outcomes};  $Q \in \Delta(\mathcal{X})$ is a prior probability distribution over $\mathcal{X}$ representing a \emph{prior policy}; and $U: \mathcal{X} \rightarrow \mathbb{R}$ is a real-valued mapping of the outcomes called the \emph{utility function}.

\paragraph{Goal.}
Given a bounded-rational decision problem $(\alpha, \mathcal{X}, Q, U)$, the \emph{goal} consists in finding the \emph{posterior policy} $P \in \Delta(\mathcal{X})$ that \emph{extremizes} the \emph{free energy functional} 
\begin{equation}\label{eq:fe}
  F[\tilde{P}] := 
    \underbrace{ \sum_x \tilde{P}(x) U(x) }_\text{Expected Utility} 
    - \frac{1}{\alpha} \underbrace{ \sum_x \tilde{P}(x) \log \frac{\tilde{P}(x)}{Q(x)} }_\text{Information Cost}.
\end{equation}
Thus, the free energy functional captures a fundamental decision-theoretic trade-off: it corresponds to the expected utility, regularized by the information cost of representing the final distribution~$P$ using the base distribution~$Q$. The functional is illustrated in Fig.~\ref{fig:free-energy-simplex}.

\paragraph{Inverse temperature.}
The inverse temperature $\alpha$ controls the trade-off between utilities and information costs by setting the exchange rate between units of information (in \emph{bits}) and units of utility (in \emph{utiles}). An agent's deliberation process can be affected by a number of disparate factors imposing information constraints; nonetheless, here the central assumption is that all of them can ultimately be condensed into the single parameter $\alpha$. 

Furthermore, notice that we have extended the domain of $\alpha$ to the real values $\mathbb{R}$. The sign of $\alpha$ determines the type of optimization: when $\alpha>0$ is positive, then the free energy functional $F_\alpha$ is concave in $\tilde{P}$ and the posterior policy $P$ is the \emph{maximizer}; and when $\alpha<0$ is negative, then $F$ is convex in $\tilde{P}$ and $P$ is the \emph{minimizer}. We will further elaborate on the meaning of the negative values later when analyzing sequential decisions.

\begin{figure}
\centering
\includegraphics[width=0.8\textwidth]{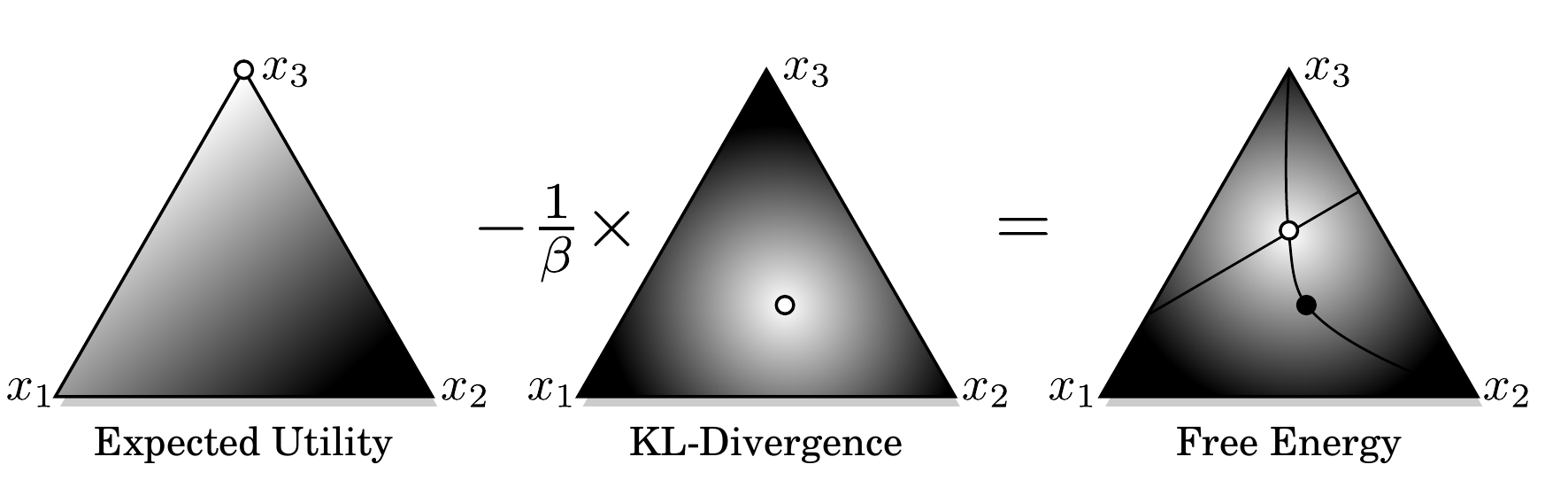}
\caption{The Free Energy Functional. A bounded-rational decision-problem combines a linear and a non-linear cost function, the first being the expected utility and the latter the KL-divergence of the posterior from the prior choice probabilities. The optimal distribution is the point on the linear subspace (defined by the expected utility and the inverse temperature) that minimizes the KL-divergence to the prior. In \emph{information geometry}, this point is known as the \emph{information projection} of the prior onto the linear subspace \citep{Csiszar2004}.}\label{fig:free-energy-simplex}
\end{figure}

\paragraph{Optimal choice.}

The optimal solution to \eqref{eq:fe} is given by the Gibbs distribution
\begin{equation}\label{eq:optimal}
  P(x) = \frac{1}{Z_\alpha} Q(x) \exp\{ \alpha U(x) \},
  \qquad Z_\alpha = \sum_x Q(x) \exp\{ \alpha U(x) \},
\end{equation}
where the normalizing constant $Z_\alpha$ is the \emph{partition function}. Inspecting \eqref{eq:optimal}, we see that the optimal choice probabilities $P(x)$ is a standard Bayesian posterior\footnote{We will further elaborate on this connection later in the text.} obtained by multiplying the prior $Q(x)$ with a \emph{likelihood} term that grows monotonically with the utility $U(x)$. The inverse temperature controls the balance between the prior $Q(x)$ and the modifier $\exp\{ \alpha U(x) \}$.

\begin{figure}[htb]
\centering
\includegraphics[width=0.8\textwidth]{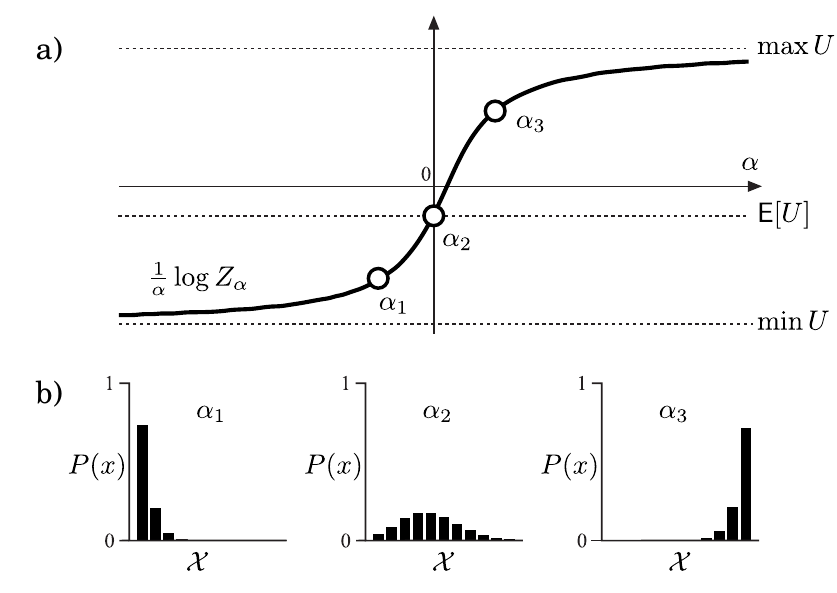}
\caption{Certainty-equivalent and optimal choice. a) The certainty-equivalent $\frac{1}{\alpha} \log Z_\alpha$, seen as a function of the inverse temperature $\alpha \in \mathbb{R}$, has a sigmoidal shape that moves between $\min U$ and $\max U$, passing through $\expect[U]$ at $\alpha = 0$. Panel (b) shows three optimal choice distributions for selected values of $\alpha$. Notice that $\alpha_2 = 0$, $P(x) = Q(x)$.}\label{fig:ce-versus-beta}
\end{figure}

\paragraph{Certainty-equivalent.}

To understand the \emph{value} that the agent assigns to a given decision problem, we need to calculate its \emph{certainty-equivalent}. To do so, we insert the optimal choice probabilities~\eqref{eq:optimal} into the free energy functional~\eqref{eq:fe}, obtaining the expression 
\begin{equation}\label{eq:certainty-equivalent}
    F := F[P]
    = \frac{1}{\alpha} \log Z_\alpha
    = \frac{1}{\alpha} \log \biggl( \sum_x Q(x) e^{\alpha U(x)} \biggr).
\end{equation}
In the text, we will always use the notation $F[\cdot]$ (with functional brackets) for the free energy and $F$ (without brackets) for the certainty-equivalent. An interesting property of the certainty-equivalent is revealed when we analyze its change with the inverse temperature $\alpha$ (Fig.~\ref{fig:ce-versus-beta}). Obviously, the more the agent is in control, the more effectively it can control the outcome, and thus the higher it values the decision problem. In particular, the value and the choice probabilities take the following limits,
\begin{align*}
  \vphantom{\sum_x}
  \alpha &\rightarrow +\infty &
    \tfrac{1}{\alpha} \log Z_\alpha
        &= \max_x U(x)
        & P(x) &= \mathcal{U}_{\max}(x)\\
  \vphantom{\sum_x}
  \alpha &\rightarrow 0 &
    \tfrac{1}{\alpha} \log Z_\alpha
        &= \sum_x Q(x) U(x)
        & P(x) &= Q(x)\\
  \vphantom{\sum_x}
  \alpha &\rightarrow -\infty &
    \tfrac{1}{\alpha} \log Z_\alpha
        &= \min_x U(x)
        & P(x) &= \mathcal{U}_{\min}(x),
\end{align*}
where $\mathcal{U}_{\max}$ and $\mathcal{U}_{\min}$ are the uniform distribution over the maximizing and minimizing subsets
\begin{align*}
  \set{X}_{\max}
  &:= \{x \in \set{X}: U(x) = \max_{x'} U(x') \}\\
  \set{X}_{\min}
  &:= \{x \in \set{X}: U(x) = \min_{x'} U(x') \}
\end{align*}
respectively. Here we see that the inverse temperature $\alpha$ plays the role of a boundedness parameter and that the single expression $\frac{1}{\alpha} \log Z$ is a generalization of the classical concept of \emph{value} in reinforcement learning (see Fig.~\ref{fig:decision-problem}).

\begin{figure}
\centering
\includegraphics[width=\textwidth]{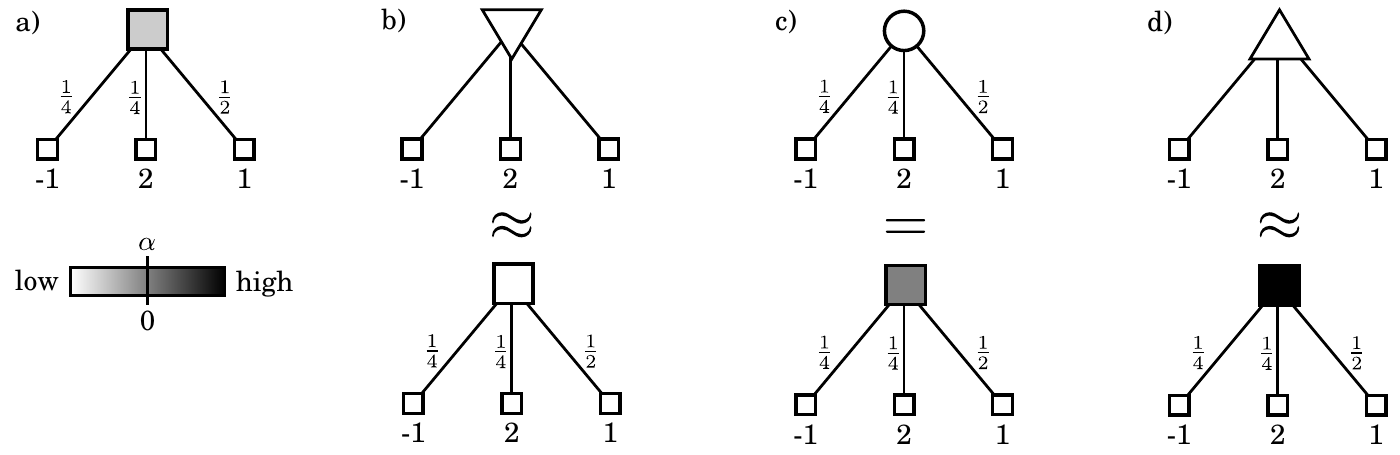}
\caption{Approximation of the classical decision rules using bounded-rational decision problems. a) We represent a bounded-rational decision problem using a colored square node, where the color encodes the inverse temperature. b--d) Classical decision rules and their bounded-rational approximations.}\label{fig:decision-problem}
\end{figure}

\subsection{Stochastic choice}\label{sec:stochastic-choice}

A perfectly-rational agent must always choose the alternative with the highest expected utility. In general, this operation cannot be done without exhaustive enumeration, which is more often than not intractable. In contrast, we expect a bounded-rational agent to inspect only a \emph{subset of alternatives} until it finds one that is good enough---that is, a \emph{satisficing} choice. Furthermore, the effort put into analyzing the alternatives should scale with the agent's level of rationality. A central feature of information-theoretic bounded rationality is that this ``algorithmic'' property of analyzing just a subset is built right into the theory. 

More precisely, this simplification is achieved by noticing that acting optimally amounts to obtaining \emph{just one random sample} from the posterior choice distribution. Any other choice scheme that does not conform to the posterior choice probabilities, such as picking the mode of the distribution for instance, violates the agent's information constraints modeled by the objective function. We review a basic sampling scheme that is readily suggested by the specific shape of the posterior distribution.

\paragraph{Rejection sampling.}

The simplest sampling scheme is immediately suggested by the form of the posterior choice distribution \eqref{eq:optimal}, illustrated in Fig.~\ref{fig:rejection-sampling}. If we interpret $Q$ as prior knowledge that is readily available to the agent in the form of random samples, then it can filter them to generate samples from $P$ using rejection sampling. This works as follows: the agent first draws a sample $x$ from $Q$ and then accepts it with probability
\begin{equation}\label{eq:rejection}
  A(x|U^\ast) = \min\bigl\{ 1, e^{ \alpha [U(x) - U^\ast] } \bigr\},
\end{equation}
where $U^\ast \in \reals$ is a target value set by the agent. This is repeated until a sample is accepted. Notice that this sampling scheme does not require the agent to draw the samples sequentially: indeed, rejection sampling can be \emph{parallelized} by drawing many samples and returning \emph{any} of the accepted choices. The next theorem guarantees that rejection sampling does indeed generate a sample from $P$.

\begin{figure}[ht]
\begin{center}
\includegraphics[width=\textwidth]{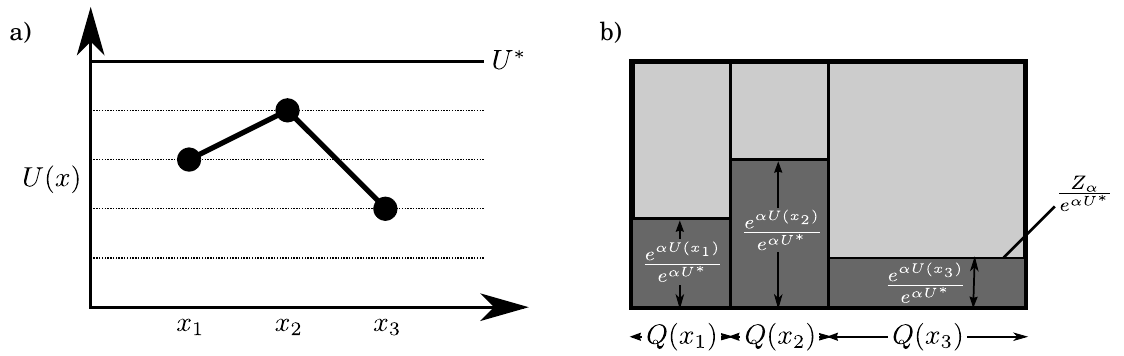}
\end{center}
\caption{Choosing optimally amounts to sampling from the posterior choice distribution $P(x) \propto Q(x) \exp\{\alpha U(x)\}$ using rejection sampling. a) A utility function $U(x)$ and a target utility $U^\ast$. b)~The choice is phrased as a search process where the agent attempts to land a sample into the target set (dark area). Generating \emph{any choice} corresponds to obtaining a successful Bernoulli random variate $Z \sim \mathcal{B}(p_\alpha)$, where the probability of success is equal to $p_\alpha = Z_\alpha / \exp\{\alpha U^\ast\}$.}
\label{fig:rejection-sampling}
\end{figure}

\begin{theorem}\label{theo:rejection}
Rejection sampling with acceptance probability~\eqref{eq:rejection} produces the correct distribution as long as $U^\ast \geq \max_x\{ U(x) \}$ when $\alpha \geq 0$ and $U^\ast \leq \min_x\{ U(x) \}$ when $\alpha \leq 0$.
\end{theorem}
\begin{proof}
First, we need a constant $c$ such that for all $x$, $P(x) \leq c \cdot Q(x)$. The smallest constant is given by
\[
  \frac{ P(x) }{ Q(x) }
  = \frac{ e^{ \alpha U(x) } }{ \sum_{x'} Q(x') e^{ \alpha U(x') } }
  \leq \frac{ e^{ \alpha U^\ast } }{ \sum_{x'} Q(x') e^{ \alpha U(x') } }
  = c.
\]
These inequalities hold whenever $U^\ast$ is chosen as $U^\ast = \max_x U(x)$ if $\alpha \geq 0$ and $U^\ast = \min_x U(x)$ if $\alpha \leq 0$. Hence, given a sample $x$ from $Q$, the acceptance probability is
\[
  \frac{ P(x) }{ c \cdot Q(x) }
  = \frac{ \frac{1}{Z} Q(x) e^{\alpha U(x)} }
     { \frac{1}{Z} Q(x) e^{\alpha U^\ast} }
    = \frac{ e^{\alpha U(x)} }{ e^{\alpha U^\ast} }.
\]
\end{proof}

\paragraph{Efficiency.}

Notice that rejection sampling is equal to the search process defined in \ref{sec:commensurable}, where the probability of the target set is equal to
\begin{equation}
 p_\alpha = \sum_x Q(x) e^{\alpha[U(x) - U^\ast]}
 = \frac{ Z_\alpha }{ e^{\alpha U^\ast} }.
\end{equation}
In other words, obtaining \emph{any} sample is equivalent to obtaining a successful Bernoulli random variate $\mathcal{B}(p_\alpha)$. Thus, the number of samples until acceptance $N_\alpha$ follows a geometric distribution $\mathcal{G}(p_\alpha)$, and the expected value is $\expect[N_\alpha] = 1/p_\alpha$. Furthermore, the number of samples $N_\alpha(\delta)$ needed so as to guarantee acceptance with a small failure probability $\delta > 0$ is given by
\begin{equation}\label{eq:required-samples}
  N_\alpha(\delta) = \biggl\lceil \frac{ \log \delta }{ \log (1-p_\alpha) } \biggr\rceil.
\end{equation}
This function is plotted in Fig.~\ref{fig:number-samples}a. 

\paragraph{Limit efficiency.} What is the most efficient sampler? Assume w.l.g.\ that the inverse temperature $\alpha$ is fixed and strictly positive. From the definition of utilities \eqref{eq:convention-prob} and the probability-complexity equivalence \eqref{eq:cost} we get
\[
  U(x) = -\cost(x \cap P|x \cap Q) + C = \frac{1}{\alpha} \log \prob(x \cap P|x \cap Q) + C.
\]
Inspecting Fig.~\ref{fig:rejection-sampling}b, we see that in rejection sampling the probability of finding any posterior choice is such that $\prob(x \cap P|x \cap Q) = \exp\{ U(x) - U^\ast \}$. Using this substitution implies 
\begin{equation}\label{eq:offset-target}
  U(x) = U(x) - U^\ast + C \qquad \Longrightarrow \qquad U^\ast = C,
\end{equation}
that is, the target utility $U^\ast$ is exactly equal to the offset $C$ that transforms complexities into utilities. Thus, choosing an inverse temperature $\alpha$ and a target utility $U^\ast$ indirectly determine the decision complexities and thereby also the probability of accepting a sample~$p_\alpha$. Since decision complexities are positive and their absolute differences fixed through the utilities, the optimal choice of the target utility must be
\begin{equation}\label{eq:most-efficient}
  U^\ast = \max_x \{ U(x)\}.
\end{equation}
Picking a smaller value for $U^\ast$ is not permitted, as it would violate the assumptions about the underlying search model (Sec.~\ref{sec:commensurable}). 

\begin{figure}[ht]
\begin{center}
\includegraphics[width=0.9\textwidth]{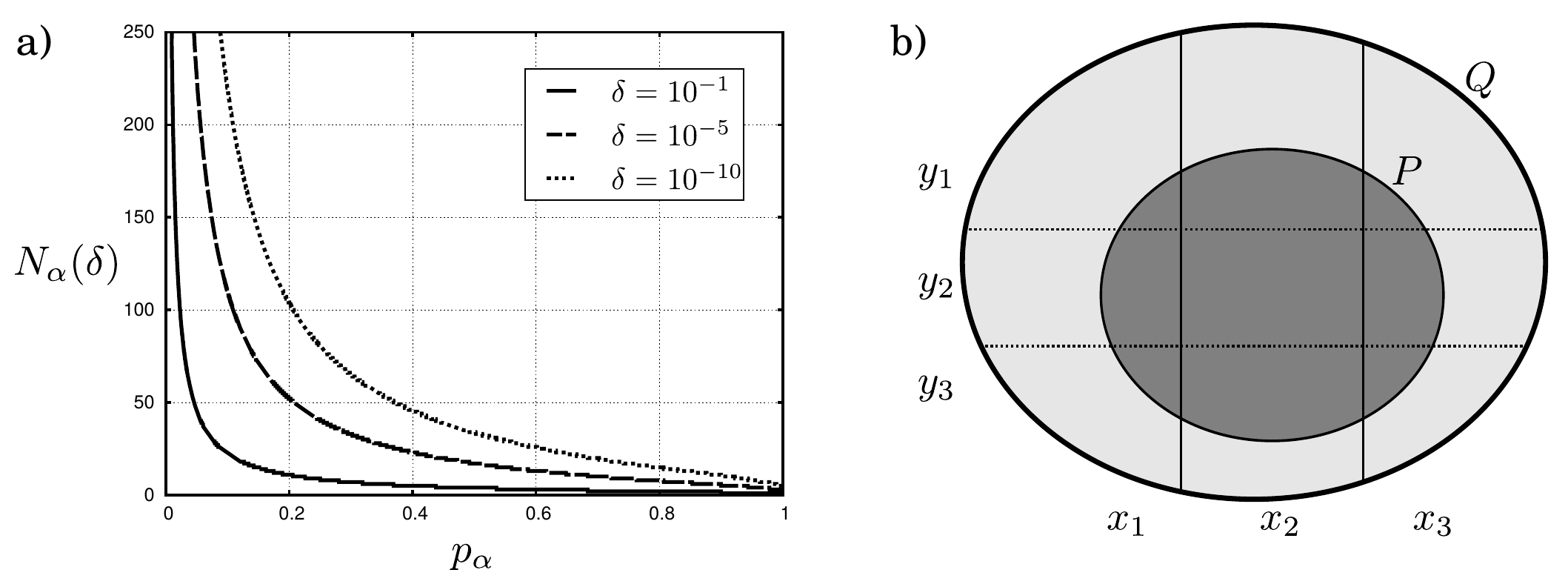}
\end{center}
\caption{a) Number of samples until acceptance. b) Granularity. Increasing the resolution of the choice set from $\mathcal{X}$ to $(\mathcal{X}\times\mathcal{Y})$ does not change the success probability.}
\label{fig:number-samples}
\end{figure}

\paragraph{Granularity.} Typically, one would expect that the number of samples to be inspected before making a decision depends on the number of available options. However, in the bounded-rational case, this is not so. Indeed, we can augment the granularity of the choice set without affecting the agent's decision effort. This is seen as follows. Assume that we extend the choice set from $\mathcal{X}$ to $(\mathcal{X} \times \mathcal{Y})$, with the understanding that every pair $(x,y) \in (\mathcal{X} \times \mathcal{Y})$ is a sub-choice of $x \in \mathcal{X}$. Then, the utilities $U(x)$ correspond to the certainty-equivalents of the utilities $U(x,y)$, that is
\begin{equation}\label{eq:granularity}
  U(x) 
  = \frac{1}{\alpha} \log Z_\alpha(x)
  = \frac{1}{\alpha} \log \biggl( \sum_{y \in x} Q(y|x) e^{\alpha U(x,y)} \biggr).
\end{equation}
Inserting this into the partition function for $x$, we get
\[
  Z_\alpha 
  = \sum_x Q(x) e^{\alpha U(x)}
  = \sum_{x,y} Q(x,y) e^{\alpha U(x,y)},
\]
that is, the partition function $Z_\alpha$ is independent of the level of resolution of the choice set. This, in turn, guarantees that the success probability of rejection sampling $p_\alpha = Z_\alpha / \exp\{ \alpha U^\ast \}$ stays the same no matter how we partition the choice set.

\subsection{Equivalence}\label{sec:equivalence}

There is more than one way to represent a given choice pattern. Two different decision problems can  lead to identical transformations of prior choice probabilities~$Q$ into posterior choice probabilities~$P$, and have the same certainty-equivalent. When these two conditions are fulfilled, we say that these decision problems are \emph{equivalent}. The concept of equivalence is important because a given decision problem can be re-expressed in a more convenient form when necessary. This will prove to be \emph{essential} when analyzing sequential decision problems later in the text.

\paragraph{Formal relation.}

Consider two equivalent bounded-rational decision problems $(\alpha, \set{X}, Q, U)$ and $(\beta, \set{X}, Q, V)$ with non-zero inverse temperatures $\alpha, \beta \neq 0$. Then, their certainty-equivalents are equal, that is,
\begin{equation}\label{eq:equivalent-ce}
  \frac{1}{\alpha} \log Z_\alpha = \frac{1}{\beta} \log Z_\beta,
\end{equation}
where the partition functions are $Z_\alpha = \sum_x Q(x) \exp\{\alpha U(x)\}$ and $Z_\beta = \sum_x Q(x) \exp\{ \beta V(x) \}$ respectively. The optimal choice probabilities of the second decision problem are equal to
\[
  P(x) 
  = \frac{1}{Z_\beta} \exp\bigl\{ \beta V(x) \bigr\} 
  = \exp\bigl\{ \beta [V(x) - \tfrac{1}{\beta} \log Z_\beta] \bigr\}
  = \exp\bigl\{ \beta [V(x) - \tfrac{1}{\alpha} \log Z_\alpha] \bigr\},
\]
where the last equality substitutes one certainty-equivalent for the other. Since these probabilities are equal to the ones of the first decision problem, we have
\[
  \exp\bigl\{ \beta[V(x) - \tfrac{1}{\alpha} \log Z_\alpha] \bigr\}
  = \exp\bigl\{ \alpha[U(x) - \tfrac{1}{\alpha} \log Z_\alpha] \}.
\]
Then, taking the logarithm and rearranging gives
\begin{equation}\label{eq:transf-ce}
  V(x) = \frac{\alpha}{\beta} U(x) + \left( \frac{1}{\alpha} - \frac{1}{\beta} \right) \log Z_\alpha.
\end{equation}
Thus, \eqref{eq:transf-ce} is an explicit formula for the relation between the inverse temperatures $\alpha, \beta$ and utilities $U(x), V(x)$ of two equivalent bounded-rational decision problems. Essentially, equivalent decision problems have utilities that are scaled versions of each other around the axis defined by the certainty-equivalent (see Fig.~\ref{fig:equivalent}). From the figure, we see that increasing the inverse temperature by a factor $c$ requires decreasing the distance of the utilities to the certainty-equivalent by a factor $1/c$, so that the product is maintained at all times.

\begin{figure}[ht]
\begin{center}
\includegraphics[width=12cm]{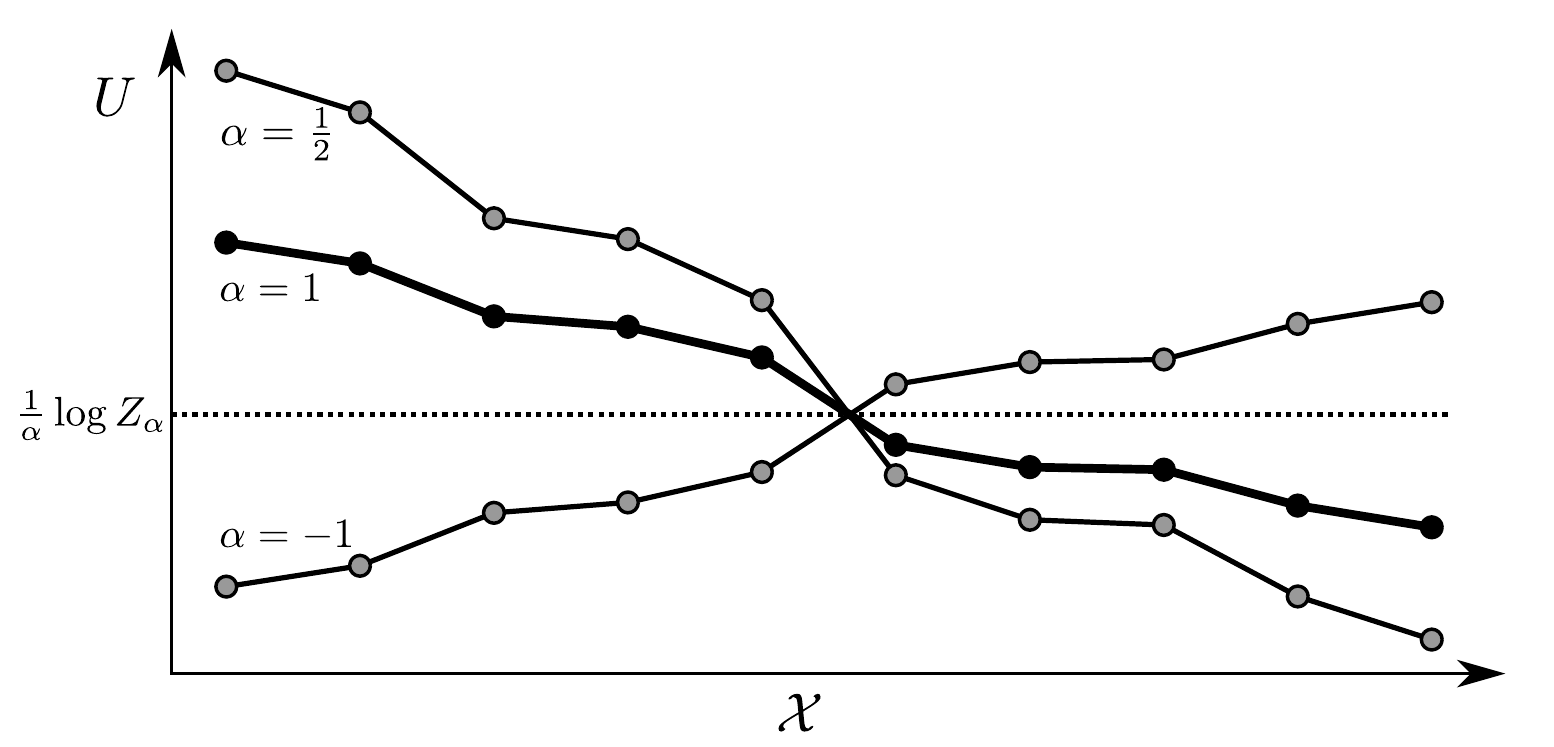}
\end{center}
\caption{Equivalent bounded-rational decision problems. The plot shows the utility curves for three equivalent decision problems with inverse temperatures $\alpha = 1, \frac{1}{2}$ and $-1$ respectively. The resulting utilities are scaled versions of each other with respect to the symmetry axis given by the certainty-equivalent $\frac{1}{\alpha}\log Z_\alpha$.}
\label{fig:equivalent}
\end{figure}

\paragraph{Sampling from equivalent decision problems.}

In rejection sampling, the success probabilities $p_\alpha$ and $p_\beta$ of two equivalent decision problems $(\alpha, \mathcal{X}, Q, U)$ and $(\beta, \mathcal{X}, Q, V)$ respectively are the same, that is,
\begin{equation}\label{eq:success-eq}
  p_\alpha = \frac{ Z_\alpha }{ e^{\alpha U^\ast} } = \frac{ Z_\beta }{ e^{\beta V^\ast} } = p_\beta,
\end{equation}
as long as their target utilities $U^\ast$ and $V^\ast$ obey the relation~\eqref{eq:transf-ce}, i.e.\ 
\begin{equation}\label{eq:conv-target}
  V^\ast = \frac{\alpha}{\beta} U^\ast 
  + \biggl( \frac{1}{\alpha} - \frac{1}{\beta} \biggr) \log Z_\alpha.
\end{equation}
An important problem is to sample directly from a decision problem $(\beta, \mathcal{X}, Q, V)$ given the inverse temperature $\alpha$ and the target utility $U^\ast$ of an equivalent decision problem. A naive way of doing so consists in using \eqref{eq:conv-target} and $\frac{1}{\alpha}\log Z_\alpha = \frac{1}{\beta} \log Z_\beta$ to derive an explicit formula for $V^\ast$:
\[
  V^\ast = \frac{\alpha}{\beta} U^\ast 
  + \biggl( 1 - \frac{\alpha}{\beta} \biggr) \frac{1}{\beta}\log Z_\beta.
\]
This formula requires integrating over the choice set to obtain the partition function~$Z_\beta$. However, this is a costly operation that does not scale to very large choice spaces: recall that all we are allowed to do is inspecting the utilities $V(x)$ for a few samples obtained from $Q(x)$. Instead, we can relate the success probability $p_\alpha$ to the samples obtained from $(\beta, \mathcal{X}, Q, V)$. This is seen by rewriting $p_\alpha$ as follows:
\[
  p_\alpha 
  = \exp\biggl\{ \alpha \Bigl[ \frac{1}{\alpha} \log Z_\alpha - U^\ast \Bigr] \biggr\}
  = \exp\biggl\{ \alpha \Bigl[ \frac{1}{\beta} \log Z_\beta - U^\ast \Bigr] \biggr\}.
\]
Then we express the inverse temperature as $\alpha = (\frac{\alpha}{\beta}) \cdot \beta$ and simplify the previous expression to
\begin{equation}\label{eq:sampling-equivalent}
  p_\alpha = p^\frac{\alpha}{\beta}, 
  \qquad \text{where} \qquad
  p := \biggl( \frac{Z_\beta}{e^{\beta U^\ast}} \biggr).
\end{equation}
This result has a convenient operational interpretation. The original problem, which consisted in obtaining one successful Bernoulli sample with success probability $p_\alpha$ has been rephrased as the problem of \emph{obtaining $\frac{\alpha}{\beta}$ successful Bernoulli samples} with probability of success~$p$. Intuitively, the reason behind this change is that the new success probability~$p$ can be larger/smaller than the original success probability $p_\alpha$, resulting in a more/less challenging search problem: therefore, \eqref{eq:sampling-equivalent} equalizes the search complexity by demanding less/more successful samples. Note that this conversion only works if the resulting rejection sampling problem fulfills the conditions of Theorem~\ref{theo:rejection}, that is: $U^\ast \geq \max_x \{V(x)\}$ for strictly positive $\beta$, or $U^\ast \leq V(x)$ for strictly negative $\beta$. To use~\eqref{eq:sampling-equivalent} effectively, we need to identify algorithms to sample an arbitrary, possibly non-integer amount of $\xi \in \mathbb{R}$ consecutive Bernoulli successes based only on a sampler for $\mathcal{B}(p)$, which in turn depends on the choice sampler $Q$. We first consider three base cases and then explain the general case.

\paragraph{Case $\xi \in \mathbb{N}$.} If $\xi$ is a natural number, then the Bernoulli random variate $Z \sim \mathcal{B}(p^\xi)$ is obtained in the obvious way by attempting to generate $\xi$ consecutive Bernoulli $\mathcal{B}(p)$ successes. If \emph{all} of them succeed, then $B$ is a success; otherwise $Z$ is a failure (see Algorithm~\ref{alg:sampling-natural}).

\begin{algoframed}
  \caption{Rejection-sampling trial for $\xi \in \mathbb{N}$}\label{alg:sampling-natural}
  \SetKwInOut{Input}{input}\SetKwInOut{Output}{output}
  \Input{A target $U^\ast$, a number $\xi \in \mathbb{N}$, and a choice sampler $Q$}
  \Output{A sample $x$ or failure $\epsilon$}
  \DontPrintSemicolon
  \For{$n=1,\ldots,\xi$}{
    Draw $x \drawnfrom Q(x)$ and $u \drawnfrom \mathcal{U}(0,1)$\;
    \lIf{$u > \exp\{ \beta [V(x) - U^\ast] \}$}{\Return $\epsilon$}
  }
  \Return $x$
\end{algoframed}

\paragraph{Case $\xi \in (0,1)$.} If $\xi$ is in the unit interval, then the Bernoulli success for $\mathcal{B}(p^\xi)$ is easier to generate than for $\mathcal{B}(p)$. Thus, we can tolerate a certain number of failures from $\mathcal{B}(p)$. The precise number is based the following theorem.

\begin{theorem}\label{theo:bernoulli}
Let $Z$ be a Bernoulli random variate with bias $(1-f_N)$ where
\[
  f_N = \sum_{n=1}^N b_n, \qquad\text{and}\qquad
  b_n = (-1)^{n+1} \frac{ \xi (\xi-1) (\xi-2) \cdots (\xi-n+1) }{ n! }
\]
for $0 < \xi < 1$ and where $N$ is a Geometric random variate with probability
of success $p$. Then, $Z$ is a Bernoulli random variate with bias $p^\xi$.
\end{theorem}

An efficient use of this sampling scheme is as follows. First, we generate an upper bound $f^\ast \sim \mathcal{U}(0,1)$ that will fix the maximum number of tolerated failures and initialize the ``trial counter'' to $f = 0$. Then, we repeatedly attempt to generate a successful Bernoulli sample $Z \sim \mathcal{B}(p)$ as long as $f < f^\ast$. If it succeeds, we return the sample; otherwise, we add a small penalty to the counter $f$ and repeat---see Algorithm~\ref{alg:sampling-unit-int}.

\begin{algoframed}
\caption{Rejection-sampling trial for $\xi \in (0,1)$}\label{alg:sampling-unit-int}
  \SetKwInOut{Input}{input}\SetKwInOut{Output}{output}
  \Input{A target $U^\ast$, a number $\xi \in (0,1)$, and a choice sampler $Q$}
  \Output{A sample $x$ or failure $\epsilon$}
  \DontPrintSemicolon
  $f^\ast \drawnfrom \mathcal{U}(0,1)$\;
  Set $b \leftarrow -1$, $f \leftarrow 0$, and $k \leftarrow 1$\;
  Draw $x \drawnfrom Q(x)$ and $u \drawnfrom \mathcal{U}(0,1)$\;
  \While{$u > \exp\{\beta[V(x)-U^\ast]\}$}{
    Set $b \leftarrow -b \cdot \frac{(\xi-k+1)}{k}$, $f \leftarrow f + b$,
    and $k \leftarrow k + 1$\;
    Draw $x \drawnfrom Q(x)$ and $u \drawnfrom \mathcal{U}(0,1)$\;
    \lIf{$f^\ast \leq f$}{\Return $\epsilon$}}
 \Return $x$\;
\end{algoframed}

\paragraph{Case $\xi = -1$.} If $\xi$ is $-1$, then the interpretation of the target utility $U^\ast$ is flipped: an upper-bound becomes a lower-bound and \textit{vice versa}. The resulting success probability is then equal to
\[
  p_\alpha = \frac{e^{\beta U^\ast}}{ Z_\beta }.
\]
We do not know how to efficiently generate a sample from the \emph{inverse} partition function. Instead, we use the following trick: we invert our acceptance criterion by basing our comparison on \emph{reciprocal probabilities} as shown in Algorithm~\ref{alg:sampling-inverse}.

\begin{algoframed}
\caption{Rejection-sampling trial for $\xi = -1$}\label{alg:sampling-inverse}
  \SetKwInOut{Input}{input}\SetKwInOut{Output}{output}
  \Input{A target $U^\ast$ and a choice sampler $Q$}
  \Output{A sample $x$ or failure $\epsilon$}
  \DontPrintSemicolon
  Draw $x \drawnfrom Q(x)$ and $u \drawnfrom \mathcal{U}(0,1)$\;
  \lIf{$1/u < \exp\{ \beta [V(x) - U^\ast] \}$}{\Return $\epsilon$}
  \Return $x$
\end{algoframed}

\paragraph{General case.} The rejection sampling algorithm for an arbitrary value of $\xi \in \mathbb{R}$ is constructed from the preceding three cases. First, the sign of $\xi$ determines whether we will base our acceptance criterion on either probabilities or reciprocal probabilities. Then, we decompose the absolute value $|\xi|$ into its integer and unit-interval parts, applying the associated algorithms to generate a choice. 

\comment{
\paragraph{Efficiency.}
\begin{mdframed}
 Explain that the number of samples in an equivalence transformation varies, but the number of bits to be generated is the same!
\end{mdframed}
} %

\subsection{Comparison}

We finish this section with a brief comparison of decision-making based upon expected utility versus free energy. Table~\ref{tab:comparison} tabulates the differences according to several criteria explained in the following. The \emph{rationality paradigm} refers to general decision-making rationale of the agent. We have seen that perfect rationality and bounded rationality can be regarded as choice models that are valid approximations to different scales or \emph{domain sizes}. The different scales, in turn, suggest different \emph{search strategies} and \emph{search scopes}: the perfect rational agent can use a fixed, deterministic rule that exploits the symmetries in the choice set in order to single out the optimal choice; while the bounded-rational agent has to settle on a satisficing choice found through random inspection---as if it were a search in a pile of straw. The two objective functions also differ in their dependency on the choice probabilities. The \emph{functional form} is such that this dependency is linear in the expected utility case and non-linear in the free energy case. Because of this, the perfectly rational agent's preferences depend only upon the expected value of the utility; whereas a bounded-rational agent also takes into account the higher-order moments of the utility's distribution (\emph{utility sensitivity}), which is easily seen through a Taylor expansion of the KL-divergence term of the free energy.

\begin{table}[ht]
\centering
\caption{Comparison of Decision Rules.}\label{tab:comparison}
\begin{tabular}{lcc}
\toprule
Objective function    & Expected Utility & Free Energy \\
\midrule 
Rationality paradigm  & perfect       & bounded      \\
Preferred domain size & small         & large        \\
Search strategy       & deterministic & randomized   \\
Search scope          & complete      & incomplete   \\
Functional form       & linear        & non-linear   \\
Utility sensitivity   & first moment  & all moments  \\
\bottomrule
\end{tabular}
\end{table}

\section{Sequential Decisions}\label{sec:sequential-decisions}

In the previous section we have fleshed out the basic theory for \emph{single-step} bounded-rational decision making. In real-world applications however, agents have to plan ahead over \emph{multiple} time steps and interact with another system called the \emph{environment}. In these \emph{sequential decision problems}, agents have to devise a \emph{policy}---that is, a decision plan---that prescribes how to act under any situation that the agent might encounter in the future.

\begin{figure}[ht]
\centering
\includegraphics[width=\textwidth]{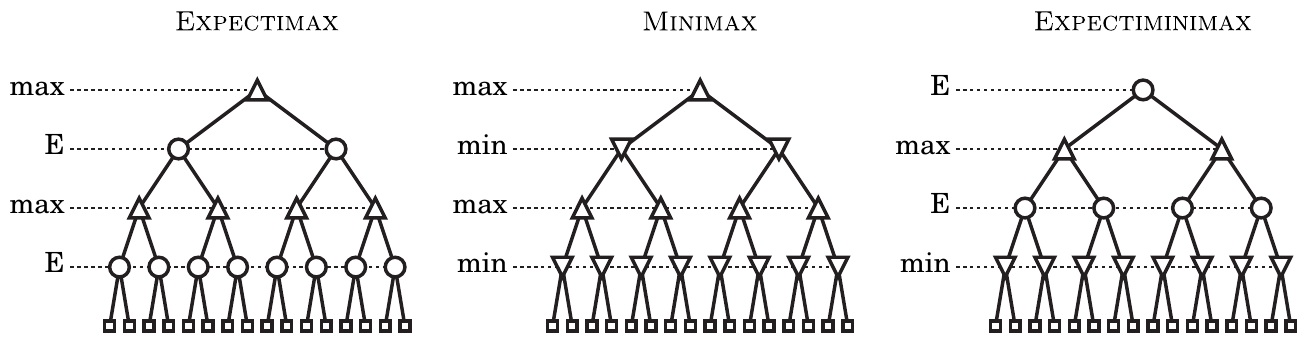}
\caption{Decision trees}\label{fig:gametrees}
\end{figure}

Like in single-step decision problems, policies are chosen using a decision rule. The classical decision rules that are used in the literature depend upon the type of system the agent is interacting with. Three such decision rules that are popular are: \textsc{Expectimax}, when the agent is interacting with a stochastic environment; \textsc{Minimax}, where the agent is playing against an adversary; and \textsc{Expectiminimax} in games with both an adversary and chance elements such as in Backgammon. The dynamical structure of the interactions between the agent and the environment is captured in a \emph{decision tree} (or a \emph{game tree}) like those illustrated in Figure~\ref{fig:gametrees}. These decision trees were built from composing the primitives $\bigtriangledown$, $\bigcirc$, and $\bigtriangleup$ (discussed in Section~\ref{sec:seu}) representing an adversarial, stochastic, and friendly transition respectively. The optimal solution is then obtained using dynamic programming\footnote{Also known as ``solving the \emph{Bellman optimality equations}'' in the control and reinforcement learning literature, and \emph{backtracking} in the economics literature.} by recursively calculating the certainty-equivalent, and then taking the transition promising the highest value. Notice that, for planning purposes, it is immaterial whether the transitions are eventually taken by the agent or by the environment: all what matters is the degree to which a particular transition contributes towards the agent's overall objective. Thus, a max-node $\bigtriangleup$ can stand both for the agent's action or another player's cooperative move for instance.

\begin{figure}[ht]
\centering
\includegraphics[width=0.9\textwidth]{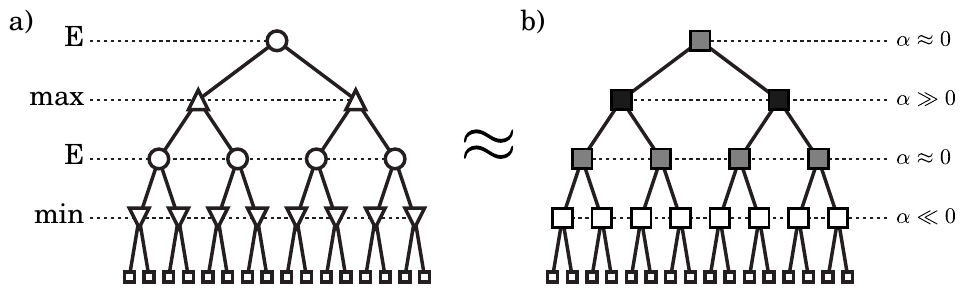}
\caption{A bounded-rational approximation. Classical decision trees, such as the \textsc{Expectiminimax} decision tree, can be approximated with a bounded-rational decision tree (b) by substituting the classical nodes with bounded-rational counterparts. Recall that the nodes are color-coded, with darker nodes corresponding to higher inverse temperatures.}\label{fig:gm-approx}
\end{figure}

These classical decision trees can be approximated by bounded-rational decision trees. Recall from Section~\ref{sec:seu} that the classical certainty-equivalent operators $\bigtriangledown$, $\bigcirc$, and $\bigtriangleup$ can be approximated by bounded-rational decision problems with appropriately chosen inverse temperatures. The very same idea can be used to substitute the nodes in a classical decision tree (see \textit{e.g.} Figure~\ref{fig:gm-approx}).

\begin{figure}[ht]
\centering
\includegraphics[width=\textwidth]{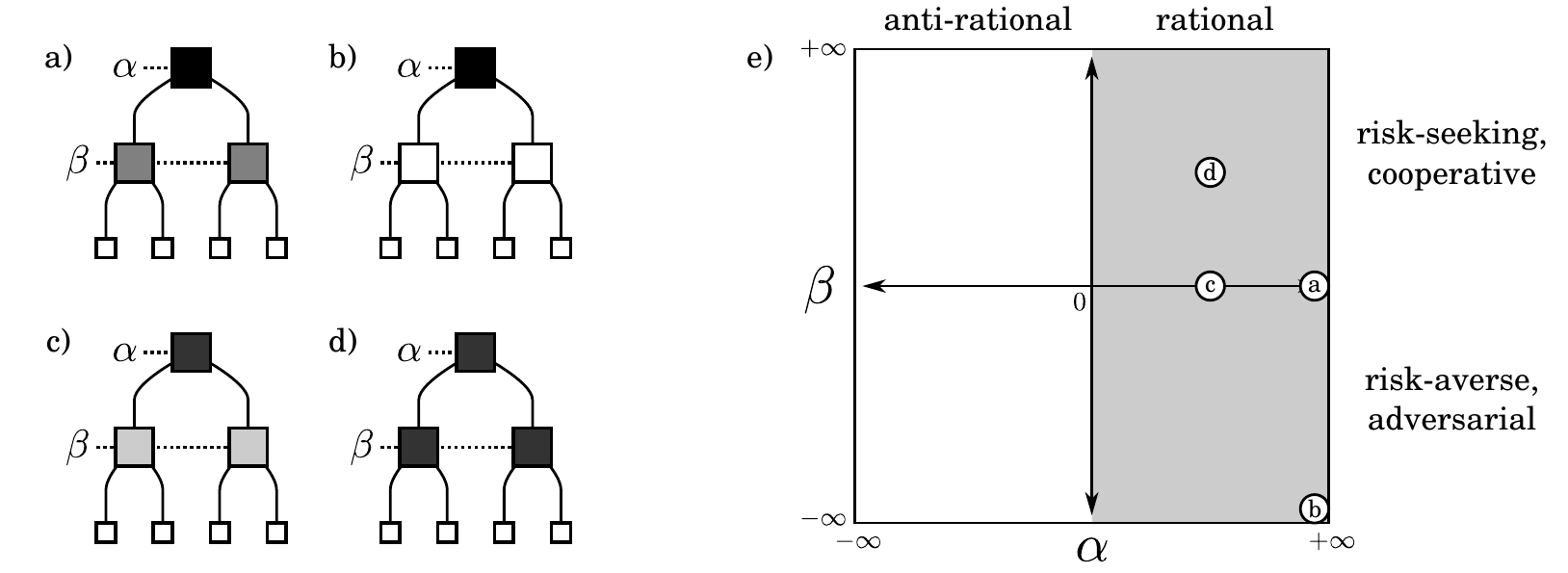}
\caption{Map of bounded-rational decision rules. Panels (a--d) depict several two-step decision problems, where the agent and environment interact with inverse temperatures $\alpha$ and $\beta$ respectively: a) case $\alpha \gg 0, \beta \approx 0$ approximates expected utility; b) case $\alpha \gg 0, \beta \ll 0$ is an approximation of a minimax/robust decision; c) case $\alpha > 0, \beta \approx 0$ is a bounded-rational control of a stochastic environment; and d) $\alpha > 0, \beta > 0$ corresponds to risk-seeking bounded-rational control. Panel~(e) shows a map of the decision rules.}\label{fig:map-of-rules}
\end{figure}

This approximation immediately suggests a much broader class of bounded-rational sequential decision problems that is interesting in its own right. The decision trees in this class are such that each node can have its own inverse temperature to model a variety of information constraints that result from \emph{resource} limitations, \emph{risk}, and \emph{trust} sensitivity. Figure~\ref{fig:map-of-rules} illustrates some example decision types. Our next goal is to formalize these bounded-rational decision trees, explain how they follow from the single-step case, and present a sampling algorithm to solve them.

\subsection{Bounded-rational decision trees}

\paragraph{Definition.}
A \emph{bounded-rational decision tree} is a tuple $(T, \set{X}, \alpha, Q, R, F)$ with the following components. $T \in \nats$ is the \emph{horizon}, \textit{i.e.}\ the depth of the tree. $\mathcal{X}$ is the \emph{set of interactions} or \emph{transitions} of the tree. We assume that it is finite, but potentially \emph{very} large. Together with the horizon, it gives rise to  an associated \emph{set of states} (or nodes) of the tree $S$, defined as
\[
  S := \bigcup_{t=0}^T \mathcal{X}^t.
\]
Each member $s \in S$ is a path of length $t \leq T$ that uniquely identifies a node in the tree. Excepting the (empty) root node $\epsilon \in S$, every other member $s' \in S$ can be reached from a preceding node $s \in S$ as long as $sx = s'$ for a transition $x \in \mathcal{X}$. The function $\alpha: S \rightarrow \mathbb{R}$ is the \emph{inverse temperature function}, and it assigns an inverse temperature $\alpha(s)$ to each node $s \in S$ in the tree. We will assume that $\alpha(s) \neq 0$ for all $s \in S$. The conditional probability distribution $Q(\cdot|\cdot)$ defines the \emph{prior transition probabilities}: thus, $Q(x|s)$ corresponds to the probability of moving to node $s'=sx \in S$ from $s \in S$ via the transition $x \in \mathcal{X}$. Obviously, $\sum_x Q(x|s) = 1$ for each node $s \in S$. Every transition has, in turn, an associated \emph{conditional reward}~$R(x|s) \in \mathbb{R}$. We assume that these rewards are additive, so that $R(uv|w) = R(u|w) + R(u|v,w)$ for any disjoint sets $u, v, w$. Finally, $F: S \rightarrow \mathbb{R}$ is the \emph{terminal certainty-equivalent function} which attaches a value $F(s)$ to each terminal state $s \in \mathcal{X}^T$. Fig.~\ref{fig:br-decision-tree} shows an example.

\begin{figure}[ht]
\centering
\includegraphics[width=\textwidth]{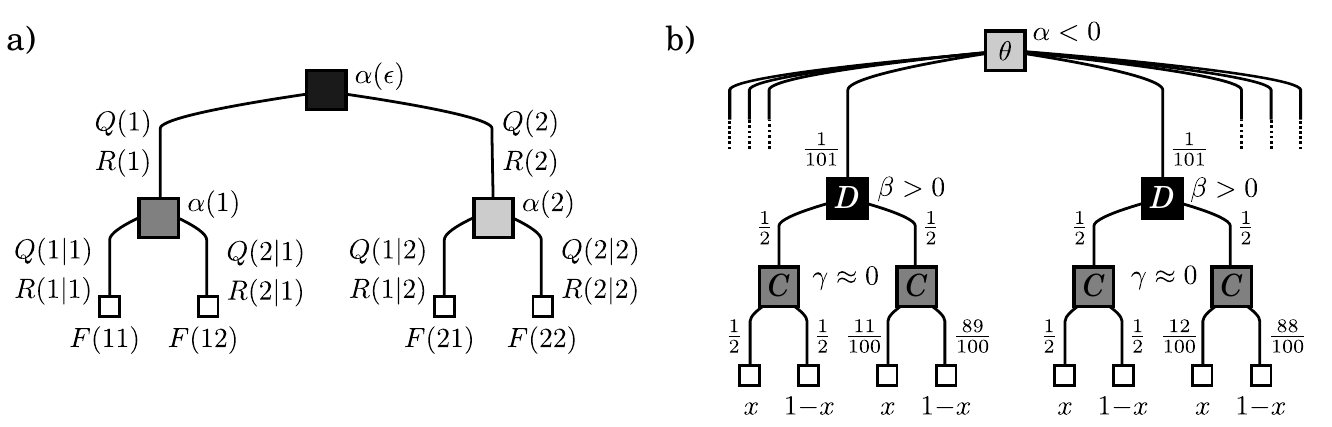}
\caption{Bounded-rational decision trees. Panel (a) shows a simple bounded-rational decision tree with two steps and $\mathcal{X} = \{1, 2\}$. These can be used to model complicated situations such as the \textit{Ellsberg Paradox} discussed in Section~\ref{sec:open-problems} shown in panel~(b). This tree models the following three random variables: $\theta$, the proportion of black \& white balls in the box (only two of its 101 children nodes are shown); $D$, the agent's bet on the left or right urn; and $C$, corresponding to the color of the ball drawn from the chosen box. We do not annotate the transitions with rewards, as we assume that they are all equal to zero. Notice that the initial choice of $\theta$ is made adversarially ($\alpha < 0$) in spite of the uniform prior over the 101 possible combinations. Because of this, the agent biases his bet towards the left box, irrespective of whether black ($x=1$) or white ($x=0$) wins.}\label{fig:br-decision-tree}
\end{figure}

\paragraph{Sequence of interactions.}
The bounded-rational decision tree models a (prior) probability distribution over random sequences $X_1, \ldots, X_T$ of length $T$. It suggests a chronological generative model: at time~$t$, \textit{i.e.}\ after having observed the $t-1$ previous interactions $X_{<t} = x_{<t}$, it draws the next interaction $X_t$ from the conditional $Q(X_t|X_{<t})$. The agent has the ability to influence the course of realization of the interaction sequence by changing the prior $Q$ into a posterior $P$. However, the extent to which it can do so depends on the state's inverse temperature.

\paragraph{Free energy.}
As in the single-step scenario, the objective function optimized by the agent is the free energy functional. For bounded-rational decision trees, the free energy functional $F[\tilde{P}]$ is given by
\begin{equation}\label{eq:fe-dt}
    F[\tilde{P}] =
    \sum_{x_{\leq T}} \tilde{P}(x_{\leq T})
        \biggl\{ \sum_{t=1}^T \biggl[
        R(x_t|x_{<t})
        - \frac{1}{\alpha(x_{<t})} \log \frac{\tilde{P}(x_t|x_{<t})}{Q(x_t|x_{<t})}
    \biggr]
    + F(x_{\leq T}) \biggr\}.
\end{equation}
Inspection of \eqref{eq:fe-dt} reveals that it is an expectation taken w.r.t.\ the realizations of the tree, where each node $s \in S$ contributes a free energy term of the form
\[ 
  \sum_{x} P(x|s) R(x|s) - \frac{1}{\alpha(s)} \sum_x P(x|s) \log\frac{ \tilde{P}(x|s) }{ Q(x|s) } 
\]
that is specific to the transition. In addition, each leaf supplies an extra term $F(s)$, which can be thought of as the certainty-equivalent of the future after time step $T$ that is not explicitly represented in the decision-tree.

\subsection{Derivation of the free energy functional} 

The free energy functional for decision trees is not arbitrary: it is \emph{derived} from the free energy functional for the single-step case \textit{via} equivalence transformations. The derivation is conceptually straightforward (see Fig.~\ref{fig:ms-cons}), although notationally cumbersome. 

\paragraph{Construction.}
Our goal is to transform a bounded-rational decision problem represented by the tuple $(\beta, \mathcal{X}^T, Q, V)$ into a bounded-rational decision tree $(T, \mathcal{X}, \alpha, Q, R, F)$, where the inverse temperature function $\alpha$ is prescribed and the rewards $R$ and $F$ are to be determined, and where the other members $T$, $\mathcal{X}$ and $Q$ stay fixed. The free energy functional of the single-step decision problem is
\begin{equation}\label{eq:fe-nested-1}
  F[\tilde{P}] =
  \sum_{x_{\leq T}} \tilde{P}(x_{\leq T}) V(x_{\leq T}) 
  - \frac{1}{\beta} \sum_{x_{\leq T}} \tilde{P}(x_{\leq T}) \log\frac{ \tilde{P}(x_{\leq T}) }{ Q(x_{\leq T}) }.
\end{equation}
This free energy with the inverse temperature $\beta$ and utility function $V$ can be restated as an equivalent bounded-rational decision problem with inverse temperature $\alpha(\epsilon)$ and utility function $U_\epsilon$, where $\alpha(\epsilon)$ is the prescribed value for the root node of the desired decision tree. The subindex in $U_\epsilon$ is just to remind us that the utility function is paired to inverse temperature $\alpha(\epsilon)$ of the root node. This yields
\begin{equation}\label{eq:fe-nested-2}
  F[\tilde{P}] =
  \sum_{x_{\leq T}} \tilde{P}(x_{\leq T}) U_\epsilon(x_{\leq T}) 
  - \frac{1}{\alpha(\epsilon)} \sum_{x_{\leq T}} \tilde{P}(x_{\leq T}) \log\frac{ \tilde{P}(x_{\leq T}) }{ Q(x_{\leq T}) }.
\end{equation}
Recall that this does not change the resulting optimal choice, as equivalent decision problems share the prior, the posterior, and the certainty-equivalent. In Fig.~\ref{fig:ms-cons}, this corresponds to the transformation from (a) to (b).

In order to introduce step-wise reinforcements of decision trees, we assume that the utilities can be recursively broken down into a sum of an instantaneous reward plus the utility of the remaining tail. For the first step, this is
\begin{equation}\label{eq:utility-additive}
  U_\epsilon(x_{\leq T}) = R(x_1) + U_\epsilon(x_{2:T}|x_1),
\end{equation}
where $R(x_1)$ is the reward of the first step and $U_\epsilon(x_{2:T}|x_1)$ is the utility of the tail $x_{2:T}$ rooted at $x_1$. How much reward we place in $R(x_1)$ is arbitrary, as long as all the paths starting with $x_1$ put the same amount of reward into this transition. Substituting this back into \eqref{eq:fe-nested-1} and rearranging yields
\begin{align}
  \nonumber
  F[\tilde{P}] 
  &= \sum_{x_{\leq T}} \tilde{P}(x_{\leq T}) \biggl\{ R(x_1) + U_\epsilon(x_{2:T}|x_1) \biggr\} 
  - \frac{1}{\alpha(\epsilon)} \sum_{x_{\leq T}} \tilde{P}(x_{\leq T}) \log\frac{ \tilde{P}(x_{\leq T}) }{ Q(x_{\leq T}) } \\
  \label{eq:fe-nested-3}
  &= \sum_{x_1} \tilde{P}(x_1) 
    \biggl\{ R(x_1) - \frac{1}{\alpha(\epsilon)}\log\frac{ \tilde{P}(x_1) }{ Q(x_1) } + F[\tilde{P}](x_1) \biggr\},
\end{align}
where we have defined the free energies $F[\tilde{P}](x_1)$ as
\begin{equation}\label{eq:fe-nested-4}
  F[\tilde{P}](x_1) = \sum_{x_{2:T}} \tilde{P}(x_{2:T}) U_\epsilon(x_{2:T}|x_1) 
  - \frac{1}{\alpha(\epsilon)} \sum_{x_{2:T}} \tilde{P}(x_{2:T}|x_1) 
    \log\frac{ \tilde{P}(x_{2:T}|x_1) }{ Q(x_{2:T}|x_1) }.
\end{equation}
Note that through this operation, we have split the original single-step decision problem over choices $x_{\leq T} \in \mathcal{X}^T$ into a two-step decision problem in which the first step is a choice among the $x_1 \in \mathcal{X}$ and the second among the tails $x_{2:T} \in \mathcal{X}^{T-1}$. That is, each initial choice~$x_1$ leads to separate single-step decision problem of the form $(\alpha(\epsilon), \mathcal{X}^{T-1}, Q(\cdot|x_1), U_\epsilon)$.

\begin{figure}[t]
\centering
\includegraphics[width=\textwidth]{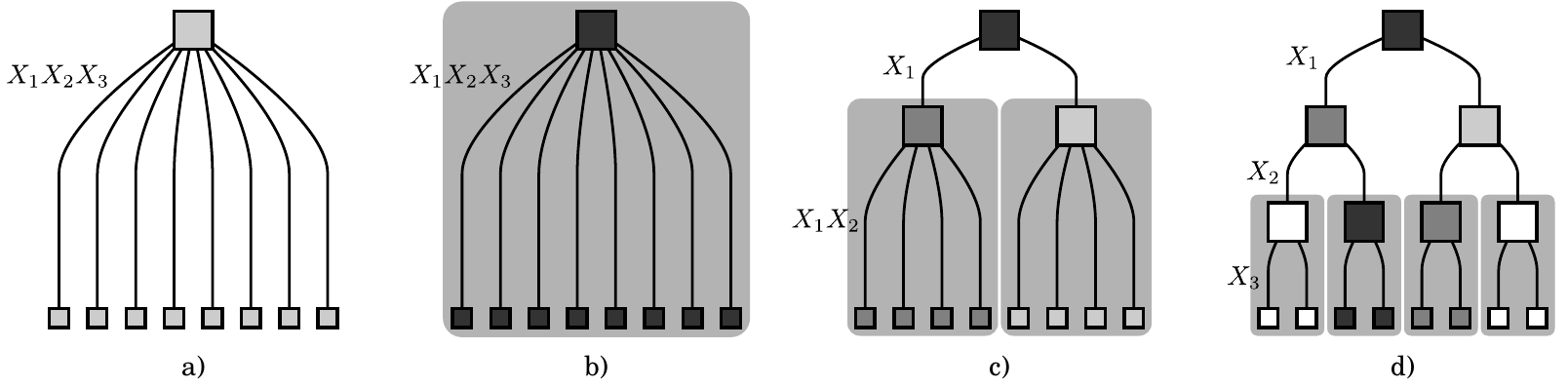}
\caption{Construction of a bounded-rational decision tree. Panel a: Starting from a (single-step) bounded-rational decision problem over a product space $\mathcal{X}^T$, we apply an equivalence transformation to set a new inverse temperature, split each choice into a transition (head) and a tail sequence, and then recur on the tails until the desired bounded-rational decision tree is fully built (Panels~b--d).}\label{fig:ms-cons}
\end{figure}

We now continue this process for $t=2,3,\ldots,T-1$, recurring in a breadth-first fashion on the nested free energies $F[\tilde{P}](x_1)$, $F[\tilde{P}](x_{\leq 2})$, \ldots, $F[\tilde{P}](x_{\leq T-2})$. Every time we do so, we first substitute the decision problem with an equivalent one having inverse temperature $\alpha(x_{\leq t})$, utility function $U_{x_{\leq t}}$, and then split the utilities into a reward $R(x_t|x_{<t})$ and a tail utility $U_{x_{\leq t}}(x_{t+1:T}|x_{\leq t})$. This results in the following expression of nested free energies:
\begin{align}
  \nonumber
  F[\tilde{P}]
  = & \sum_{x_1} \tilde{P}(x_1) 
    \biggl\{ R(x_1) - \frac{1}{\alpha(\epsilon)}\log\frac{ \tilde{P}(x_1) }{ Q(x_1) } \\
    \nonumber
    & + \sum_{x_2} \tilde{P}(x_2|x_1)
    \biggl\{ R(x_2|x_1) - \frac{1}{\alpha(x_1)}\log\frac{ \tilde{P}(x_2|x_1) }{ Q(x_2|x_1) } \\
    \nonumber
    & \phantom{+} \vdots \\
    \label{eq:fe-nested-5}
    & \phantom{+} + \sum_{x_{T-1}} \tilde{P}(x_{T-1}|x_{<T-1})
    \biggl\{ F[\tilde{P}](x_{<T}) \biggr\} \cdots \biggr\}\biggr\}.
\end{align}
Finally, for the base case of the recursion given by the collection of the innermost free energies $F[\tilde{P}](x_{<T})$, which are equal to
\[
  F[\tilde{P}](x_{<T}) = \sum_{x_T} P(x_T|x_{<T}) U_{x_{<T}}(x_T|x_{<T}) 
  - \frac{1}{\alpha(x_{<T})} \sum_{x_T} P(x_T|x_{<T}) \log \frac{ \tilde{P}(x_T|x_{<T}) }{ Q(x_T|x_{<T}) },
\]
we split the utilities arbitrarily into a reward $R(x_T|x_{<T})$ and a terminal free energy $F(x_{\leq T})$, that is,
\[
  U_{x_{<T}}(x_T|x_{<T}) = R(x_T|x_{<T}) + F(x_{\leq T}).
\]
Inserting this back into \eqref{eq:fe-nested-5} and rearranging the probabilities gives the desired free energy functional for the bounded rational decision tree \eqref{eq:fe-dt}. 

\paragraph{Insights from the construction.}

This derivation shows that the multi-step case does not require introducing any additional theory---everything can be reduced to the building blocks already developed in Section~\ref{sec:single-step} for the single-step case.

It is worthwhile pointing out that the sequence of steps we took in order to derive the decision tree's free energy functional from the single-step decision problem is completely reversible. It is easy to see how to extend this procedure to formulate an algorithm that takes any source decision-tree and transforms it into another equivalent decision tree with a prescribed inverse temperature function. In other words, every bounded-rational decision tree $(T, \set{X}, \alpha, Q, R, F)$ can be transformed into another one $(T, \set{X}, \alpha', Q, R', F')$ with an arbitrarily chosen inverse temperature function $\alpha'$ as long as both $\alpha$ and $\alpha'$ are non-zero. In this case, we say that these decision trees are equivalent.

This equivalence relation induces a quotient space (\textit{i.e.} a set of equivalence classes) on the set of bounded-rational decision trees that has rather counterintuitive properties. For instance, it turns out that a bounded-rational \textsc{Minimax} decision-problem can be turned into an equivalent anti-rational, risk-seeking decision-problem, where the agent picks an deleterious action in the hope that the environment will save it!

\subsection{Bellman recursion and its solution}

Given a bounded-rational decision tree $(T, \mathcal{X}, \alpha, Q, R, F)$, we can express its associated free energy functional \eqref{eq:fe-dt} as a Bellman recursion. The Bellman recursion has the advantage of offering a simplified, \emph{functional} view of the sequential planning problem that facilitates the characterization of the optimal policy. 

\paragraph{Free energy.}
The recursive form of the free energy functional is readily obtained from the nested form of the free energy~\eqref{eq:fe-nested-5}. First, for each leaf node $s \in S$ we equate the free energy functional to the certainty-equivalent, that is, $F[\tilde{P}](s) = F(s)$. These correspond to the base cases of the recursion. Then, for each internal node $s \in S$, we have
\begin{equation}\label{eq:fe-rec}
  F[\tilde{P}](s) = \sum_{x} P(x|s) \biggl[ R(x|s) 
    - \frac{1}{\alpha(s)} \log \frac{ \tilde{P}(x|s) }{ Q(x|s) }
    + F[\tilde{P}](sx) \biggr].
\end{equation}
This was obtained by exploiting the self-similar structure of~\eqref{eq:fe-nested-5}.

\paragraph{Certainty-equivalent.}

Before we state the optimal solution, it is convenient to first present the recursive characterization of the certainty-equivalent. For this, consider an internal node $s \in S$. Once the certainty-equivalents $F(sx)$ of its children nodes become available, the multi-step decision problem rooted at $s$ reduces to a single-step decision problem $(\alpha(s), \mathcal{X}, Q(\cdot|s), U_s)$ where the utility function is given by the sum $U_s(x) := R(x|s) + F(sx)$ of the instantaneous reward and the summarized future rewards for all $x \in \mathcal{X}$. Hence, using the formula for the certainty-equivalent \eqref{eq:certainty-equivalent} we get  
\begin{equation}\label{ce-rec}
 F(s) = \frac{1}{\alpha(s)} \log \sum_x Q(x|s) 
  \exp\biggl\{ \alpha(s) \Bigl[ R(x|s) + F(sx) \Bigr] \biggr\}.
\end{equation}
This can be regarded as an operator on the transition probabilities, rewards and free energies of the child nodes (Figure~\ref{fig:smp-rec}a).

\begin{figure}[t]
\centering
\includegraphics[width=\textwidth]{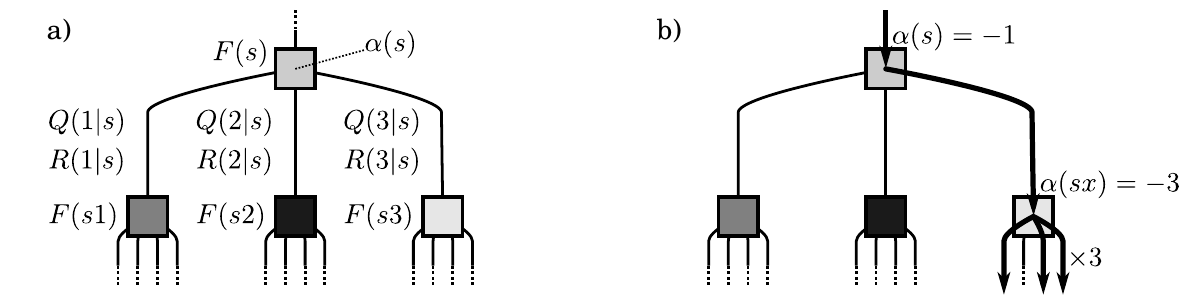}
\caption{a) The certainty-equivalent $F(s)$ of the node $s \in S$ is calculated from the transition probabilities $Q(x|s)$, immediate rewards $R(x|s)$, and certainty-equivalents $F(sx)$ of its children nodes. b) When a path's inverse temperature changes from $\alpha(s)$ to $\alpha(sx)$, then the sampler must equalize this transition by requiring $\alpha(sx)/\alpha(s)$ successful samples from the subtree rather than just one. }\label{fig:smp-rec}
\end{figure}

\paragraph{Optimal solution.} Similarly, given an internal node $s \in S$ and the certainty-equivalents $F(sx)$ of its children nodes, the optimal decision at $s$ is obtained immediately from the equation for the single-step case~\eqref{eq:optimal}:
\begin{equation}\label{eq:opt-rec}
    P(x|s) = \frac{1}{Z(s)}
        Q(x|s) \exp\biggl\{
        \alpha(s) \Bigl[ R(x|s)
            + F(sx)
            \Bigr] \biggr\},
\end{equation}
where $Z(s)$ is the partition function which normalizes the probabilities. Notice that the immediate rewards $R(x|s)$ (given $\alpha(s)$ and the $Q(x|s)$) do not form an alternative representation of the probabilities $P(x|s)$. Rather, a complete specification requires taking the sums $R(x|s) + F(sx)$ as primitives for the representation.

\subsection{Recursive rejection sampling}

Similarly to the single-step decision case, choosing an optimal action amounts to sampling from the optimal choice distribution~\eqref{eq:opt-rec}. The naive way of doing this consists in computing~\eqref{eq:opt-rec} recursively and then to sample from it. A second, slightly more elaborate way, consists in transforming the original bounded-rational decision tree into a single-step decision problem by setting every node's inverse temperature to a single value $\beta$ using equivalence transformations, and then to sample an optimal choice using rejection sampling as explained in Section~\ref{sec:stochastic-choice}. Unfortunately, both aforementioned methods depend on computing the free energies~$F(sx)$, and this operation does not scale to very large choice spaces.

Instead, we can sample a path $\tau \in \mathcal{X}^T$ from \eqref{eq:opt-rec} using a recursive form of rejection sampling without ever calculating a free energy. To do so, we first make the following two observations that follow immediately from the single-step case. First, the probability of obtaining \emph{any} successful sample using rejection sampling is a Bernoulli random variable with success probability
\begin{equation}\label{eq:sp-rec}
  \biggl( \frac{Z(s)}{e^{\alpha(s) U^\ast(s)}} \biggr)
  = \sum_x Q(x|s) \frac{ \exp\Bigl\{ \alpha(s) \Bigl[R(x|s) + F(sx)\Bigr] \Bigr\} }{ \exp\Bigl\{ \alpha(s) U^\ast(s) \Bigr\} },
\end{equation}
where $U^\ast(s)$ is a given target utility for node $s \in S$ that has to be either larger than $\max_x \{R(x|s) + F(sx)\}$ if $\alpha(s) > 0$ or smaller than $\min_x \{R(x|s) + F(sx)\}$ if $\alpha(s) < 0$ respectively.
Second, the free energy and the partition function are related as
\begin{equation}\label{eq:fe-pf}
  F(s) = \frac{1}{\alpha(s)} \log Z(s).
\end{equation}
If $s \in S$ is an internal node, then we can use the latter equation to substitute the free energies in expression \eqref{eq:sp-rec}. This in turn allows us to rewrite the success probability recursively:
\begin{align}
  \nonumber
  \biggl( \frac{Z(s)}{e^{\alpha(s) U^\ast(s)}} \biggr)
  &= \sum_x Q(x|s) \frac{ \exp\Bigl\{ \alpha(s) \Bigl[R(x|s) + \frac{1}{\alpha(sx)} \log Z(sx)\Bigr] \Bigr\} }{ \exp\Bigl\{ \alpha(s) U^\ast(s) \Bigr\} } \\
  \nonumber
  &= \sum_x Q(x|s) \frac{ \exp\Bigl\{ \alpha(s) \Bigl[\frac{1}{\alpha(sx)} \log Z(sx)\Bigr] \Bigr\} }{ \exp\Bigl\{ \alpha(s) \Bigl[ U^\ast(s) - R(x|s)\Bigr] \Bigr\} } \\
  \label{eq:samp-rec}
  &= \sum_x Q(x|s) \biggl( \frac{Z(sx)}{e^{\alpha(sx) U^\ast(sx)}} \biggr)^\frac{ \alpha(sx) }{ \alpha(s) },
\end{align}
where we have defined $U^\ast(sx) := U^\ast(s) - R(x|s)$ as the target utility for the subtree rooted as $sx \in S$. Thus, \eqref{eq:samp-rec} says that obtaining a sample from an internal node $s \in S$ amounts to first picking a random transition $x \in \mathcal{X}$ and then obtaining $\alpha(sx)/\alpha(s)$ successful samples from the subtree rooted at $sx \in S$, which we know is an equivalent subtree having inverse temperature $\alpha(sx)$---see Figure~\ref{fig:smp-rec}b. We have already seen in Section~\ref{sec:equivalence} how to sample this operation using (modifications of) the rejection sampling algorithm. If $s \in S$ is a terminal node instead, then we can treat it as a normal single-step decision problem and sample a choice using rejection sampling.

\begin{figure}[p] 
\begin{algoframed}
  \caption{\textsc{sample}($s, U^\ast, \sigma$)}\label{alg:smp-rec}
 \IncMargin{2cm}
 \DontPrintSemicolon
 \SetKwInOut{Input}{input}\SetKwInOut{Output}{output}
 \Input{A state $s$, a target utility $U^\ast$, and a sign $\sigma$}
 \Output{An accepted trajectory $\tau$, otherwise $\epsilon$}

 \BlankLine
 \emph{Proposal:}\;
 $x \drawnfrom Q(x|s)$\;
 \BlankLine
 \emph{Base case:}\;
 \If{$sx$ is a leaf}{
  $u \drawnfrom \mathcal{U}(0,1)$\;
  $p \leftarrow \exp\{\alpha(s) (R(x|s)+F(sx)-U^\ast)\}$\;
  \lIf{$(\sigma > 0$ and $u \leq  p)\phantom{1/}$}{\Return $sx$}
  \lIf{$(\sigma < 0$ and $1/u \geq p)$}{\Return $sx$}
  \Return $\epsilon$\;
 }

 \BlankLine
 \emph{Recursion:}\;
 $\sigma \leftarrow \sigma \cdot \mathrm{sign}(\alpha(s)/\alpha(sx))$\;
 $\xi \leftarrow \mathrm{abs}(\alpha(s)/\alpha(sx))$\;

 \BlankLine
 \Begin(\emph{Attempt to generate $\lfloor \xi \rfloor$ successful samples.}){
 \For{$1, \ldots, \lfloor \xi \rfloor$}{
  $\tau \leftarrow \textsc{sample}(sx, U^\ast - R(x|s), \sigma)$\;
  \lIf{$\tau = \epsilon$}{\Return $\epsilon$}
 }
 \lIf{$\xi = \lfloor \xi \rfloor$}{\Return $\tau$}
 }
 
 \BlankLine
 \Begin(\emph{Attempt to generate $\xi - \lfloor \xi \rfloor$ successful samples.}){
 $u \drawnfrom \mathcal{U}(0,1)$\;
 $a \leftarrow \lfloor \xi \rfloor - \xi$\;
 set $b \leftarrow -1$, $f \leftarrow 0$, and $k \leftarrow 1$\;
 $\tau \leftarrow \textsc{sample}(sx, U^\ast - R(x|s), \sigma)$\;
 \While{$\tau = \epsilon$}{
  set $b \leftarrow -b \cdot (a-k+1) / k$, $f \leftarrow f + b$,
  and $k \leftarrow k + 1$\;
  $\tau \leftarrow \textsc{sample}(sx, U^\ast - R(x|s), \sigma)$\;
  \lIf{$u \leq f$}{\Return $\epsilon$}
 }
 \Return $\tau$\;
 } 
\end{algoframed}
\end{figure}

A detailed pseudocode for this recursive rejection sampling algorithm is listed in Algorithm~\ref{alg:smp-rec}. The function \textsc{sample}($s, U^\ast, \sigma$) takes as arguments a node $s \in S$; a target utility $U^\ast$ for the whole path; and a sign-flag $\sigma \in \{-1,+1\}$, which keeps track of whether we are using normal ($\sigma = +1$) or reciprocal probabilities ($\sigma = -1$) in the rejection sampling step due to the possible change of sign in the inverse temperature. It returns either an accepted path $\tau \in \mathcal{X}^T$ or $\epsilon$ if the proposal is rejected. Planning is invoked on the root node by executing \textsc{sample}($\epsilon, U^\ast, 1$), where $U^\ast$ is the global target utility. For this recursive sampler to return a path drawn from the correct distribution, the value of $U^\ast$ must be chosen so as to be larger than the rewards of any path, \textit{i.e.} for all $x_{\leq T} \in \mathcal{X}^T$,
\[
  U^\ast \geq \sum_{t=1}^T R(x_t|x_{<t}) + F(x_{\leq T})
\]
whenever the root node's inverse temperature $\alpha(\epsilon) > 0$ is strictly positive; or smaller than the rewards of any path whenever $\alpha(\epsilon) < 0$. 

Finally we note that, as in the single-step case, the agent can parallelize the recursive rejection sampler by simultaneously exploring many stochastically generated paths, and then return \emph{any} of the accepted ones. 

\section{Discussion}

\subsection{Relation to literature}

In this paper we have presented a summary of an information-theoretic model of bounded-rational decision-making that has precursors in the economic literature \citep{McKelvey1995, Mattsson2002, Wolpert2004, Sims2003, Sims2005, Sims2006, Sims2011} and that has emerged through the application of information-theoretic methods to stochastic control and to perception-action systems \citep[see \textit{e.g.}][]{Mitter2005, Kappen2005a, Todorov2006, Todorov2009, Theodorou2010, Theodorou2015, Still2009, Still2012, Broek2010, Friston2010, Peters2010, Tishby2011, Kappen2012, Neumann2012,  Rawlik2012, Rubin2012, Fox2012, Neumann2013, GrauMoya2013, Zaslavsky2015, Tanaka2015, Mohamed2015, Genewein2015, Leibfried2015}. In particular, the connection between the free energy functional and bounded rationality, the identification of the free energy extremum as the certainty-equivalent in sequential games, and the implementation of the utility-complexity trade-off based on sampling algorithms, were developed in a series of publications within this community \citep[see \textit{e.g.}][]{BraunOrtega2011, OrtegaBraun2011b, OrtegaBraun2013, Ortega2014a, OrtegaKimLee2015}\nocite{Ortega2011, OrtegaBraun2012, BraunOrtega2014, OrtegaLee2014}. The most distinguishing feature of this approach to bounded rationality is the formalization of resources in terms of information.

\paragraph{Historical roots.}
The problem of bounded rational decision-making gained popularity in the 1950s originating in the work by Herbert Simon \citep{Simon1956, Simon1972, Simon1984}. Simon proposed that bounded rational agents do not optimize, but \textit{satisfice}---that is, they do not search for the absolute best option, but rather settle for an option that is good enough. Since then the research field for bounded rational decision-making has considerably diversified leading to a split between optimization-based approaches and approaches that dismiss optimization as a misleading concept altogether \citep{Lipman1995, Russell1995a, Russell1995b, Aumann1997, Rubinstein1998, Gigerenzer2001}. In particular, the formulation as a constrained optimization problem is argued to lead to an infinite regress, and the paradoxical situation that a bounded rational agent would have to solve a more complex (\textit{i.e.} constrained) optimization problem than a perfectly rational agent. Information-theoretic bounded rationality provides a middle ground here, as the agent randomly samples choices, but does not incur into a meta-optimization, because the random search simply stops when resources run out. In fact, the equations for information-theoretic bounded rational agent can be interpreted as a stochastic satisficing procedure instantiated by rejection sampling.

\paragraph{KL-Control.}
From the vantage point of an external observer, information-theoretic bounded rational agents appear to trade off any gains in utility against the additional information-theoretic complexity as measured by the KL-divergence between a prior decision strategy and a posterior decision strategy after deliberation. Recently, a number of studies have suggested the use of the relative entropy as a cost function for control, which is sometimes referred to as \textit{KL-Control} \citep{Todorov2006, Todorov2009, Kappen2012}. In the work by \citet{Todorov2006}, the transition probabilities of a Markov decision process are controlled directly, and the control costs are given by the KL-divergence between the controlled dynamics and the passive dynamics described by a baseline distribution. This framework has also been extended to the continuous case, leading to the formulation of path integral control \citep{Kappen2005a, Theodorou2010}. Conceptually, the most important difference to the bounded-rational interpretation is that in KL-Control the stochasticity of choice is thought to arise from environmental passive dynamics rather than being a direct consequence of limited information capacities.

\paragraph{Psychology.}
In the psychological and econometric literature, stochastic choice rules have extensively been studied starting with \citet{Luce1959}, extending through \citet{McFadden1974}, \citet{Meginnis1976}, \citet{Fudenberg1993}, \citet{McKelvey1995} and \citet{Mattsson2002}. The vast majority of models has concentrated on \textit{logit choice models} based on the \textit{Boltzmann distribution} which includes the \textit{softmax rule} that is popular in the reinforcement learning literature \citep{Sutton1998}. \citet{McFadden1974} has shown that such Boltzmann-like choice rules can arise, for example, when utilities are contaminated with additive noise following an extreme value distribution. From the physics literature, it is also well-known that Boltzmann distributions arise in the context of variational principles in the free energy implying a trade-off between utility and entropic resource costs. The information-theoretic model of bounded rationality generalizes these ubiquitous logit choice models by allowing arbitrary prior distributions, as for example in \citet{McKelvey1995}. This corresponds to a trade-off in utility gains and additional resource costs incurred by deviating from the prior, which ultimately entails a variational principle in a free energy difference.

\paragraph{Economics.}
In the economic literature, variational principles for choice have been suggested in the context of \textit{variational preference models} \citep{Rustichini2006}. In variational preference models the certainty-equivalent value of a choice consists of two terms: the expected utility and an ambiguity index. A particular instance of the variational preference model is the \textit{multiplier preference model} where the ambiguity index is given by a Kullback-Leibler divergence \citep{Hansen2008}. In particular, it has been proposed that multiplier preference models allow dealing with model uncertainty, where the KL-divergence indicates the degree of model uncertainty. Free energy variational principles also appear in variational Bayesian inference. In this case the utility function is given by the negative log-likelihood and the free energy trade-off captures the transformation from Bayesian prior to posterior. The variational Bayes framework has recently also been proposed as a theoretical framework to understand brain function \citep{Friston2009, Friston2010} where perception is modeled as variational Bayesian inference over hidden causes of observations.

\paragraph{Computational approaches.}
While the free energy functional does not depend on domain-specific assumptions, the exact relationship between information-theoretic constraints and the standard measures of algorithmic complexity in computer science (\textit{e.g.} space and time) is not known. In contrast, the notion of \textit{bounded optimality} defines the optimal policy as the program that achieves the highest utility score on a particular machine given complexity constraints, and oftentimes relies on meta-reasoning for practical implementations \citep{Horvitz1989, Russell1995b}. This view has recently experienced a revival in computational neuroscience under the name of \textit{computational rationality} \citep{Lieder2014, Lewis2014, Griffiths2015, Gershman2015}. Another recent approach to model bounded resources with profound implications is \textit{space-time embedded intelligence} in which agents are treated as \textit{local computation patterns} within a global computation of the world \citep{Orseau2012}. It remains an interesting challenge for the future to extend the framework of information-theoretic bounded rationality to the realm of programs and to relate it to notions of algorithmic complexity. 

\comment{%
\subsection{Relationship to Bayesian inference}

\subsection{Relationship to minimax, Markovitz \& regret}

Real-life decisions are risk-sensitive. For instance, expected utility is indifferent between two
choices as long as they have the same expected utility. However, real-life investors also consider higher-order moments when designing their portofolios; such optimal portofolios mix assets in order to balance returns and risks.

Modern portfolio theory: One notable case that is widely used in practice is modern portfolio theory. Here, an investor trades off asset resutns versus porfolio risk, encoded into a regularisation term that is quadratic in the policy \citep{Markovitz1952}. The objective function blabla

Regret theory \citep{Fishburn1982, Bell1982, LoomesSugden1982} .

\subsection{Bounded-rational games}
} 

\subsection{Conclusions}

The original question we have addressed is: how do agents make decisions in very large and unstructured choice spaces? The need for solving this question is becoming increasingly critical, as these decision spaces are ubiquitous in modern agent systems. To answer this, we have formalized resource limitations as information constraints, and then replaced the objective function of subjective expected utility theory with the free energy functional. An advantage of the free energy functional is that its optimal solution has a clear operational interpretation. As a result, the optimal solution is a fully parallelizable stochastic choice strategy that strikes a trade-off between the utility and the search effort. 

Perhaps more fundamentally though, the theory lays out a general method to model reasoning under information constraints that arise \emph{as symptoms} of intractability, model uncertainty, or other causes. This feature becomes especially apparent in the sequential decision case, where a bounded-rational decision-tree captures an agent's dynamics of trust---both in its own ability to shape \& predict the future, and in the other players' intentions. We have seen that model uncertainty biases the value estimates of an agent, forcing it to pay attention to the higher-order moments of the utility.

For the sake of parsimony of the exposition, in this work we have refrained from elaborating on specific applications or extensions. There are obvious connections to Bayesian statistics that we have not fleshed out. Furthermore, the ideas outlined here can be applied to any choice process that is typically subject to information constraints, among them: attention focus and generation of random features; model selection and inference in probabilistic programming; and planning in active learning, Bayesian optimization, multi-agent systems and partially-observable Markov decision processes. The success of the theory will ultimately depend on its usefulness in these application domains.


\subsection*{Acknowledgments}

The authors would like to thank Daniel Polani, Bert J. Kappen, and Evangelos Theodorou, who provided innumerable insights during many discussions. This study was funded by the Israeli Science Foundation Center of Excellence, the DARPA MSEE Project, the Intel Collaborative Research Institute for Computational Intelligence (ICRI-CI), and the Emmy Noether Grant BR 4164/1-1.


\appendix

\section{Proofs}

\subsection{Proof to Theorem~\ref{theo:bernoulli}}
\begin{proof}
By expanding $h(p) = p^\xi$ around $p_0=1$, Taylor's theorem asserts that
\[
  p^\xi = 1 - \xi (1-p) + \frac{\xi (\xi-1)}{2!} (1-p)^2
    - \frac{\xi (\xi-1) (\xi-2)}{3!} (1-p)^3 + \cdots
\]
Since $0<\xi<1$, each term after the first is negative, hence
\begin{equation}\label{eq:bern1}
  p^\xi = 1 - \sum_{n=1}^\infty b_n (1-p)^n
  \qquad\text{where}\qquad
  b_n := (-1)^{n+1} \frac{ \xi (\xi-1) (\xi-2) \cdots (\xi-n+1) }{ n! }
\end{equation}
and where the $0 \leq b_n \leq 1$ are known a priori. Hence,
\begin{align*}
  1 - p^\xi
  &= \sum_{n=1}^\infty b_n (1-p)^n
  = \sum_{n=1}^\infty b_n \sum_{k=0}^\infty (1-p)^{n+k} p
  = \sum_{n=1}^\infty \left( \sum_{k=0}^n b_k \right) (1-p)^n p,
\end{align*}
where the second equality follows from multiplying the term $(1-p)^n$ with
\[
  1 = p \cdot p^{-1} = p \sum_{k=0}^\infty (1-p)^k
\]
and the third from a diagonal enumeration of the summands. Define $f_0 := 0$ and
$f_n := \sum_{k=0}^n b_k$ for $n \geq 1$. Since from~\eqref{eq:bern1},
\[
  1 = 1 - 0^\xi = \sum_{n=1}^\infty b_n,
\]
we know that $0 \leq f_n \leq 1$ for all $n \geq 0$ as well. Finally,
\[
  p^\xi = 1 - \sum_{n=1}^\infty f_n (1-p)^n p
  = \sum_{n=1}^\infty (1-f_n) (1-p)^n p
\]
corresponds to the expectation of $(1-f_n)$, where $n$ follows a Geometric
distribution with probability of success $p$. To obtain a Bernoulli random
variable $u$ with bias $p^\xi$, we can sample $n$ first and
then generate a Bernoulli random variable $u|n$ with bias $1-f_n$.
\end{proof}


\vskip 0.2in
\bibliographystyle{plainnat}
\bibliography{bibliography}

\begin{thebibliography}{87}
\providecommand{\natexlab}[1]{#1}
\providecommand{\url}[1]{\texttt{#1}}
\expandafter\ifx\csname urlstyle\endcsname\relax
  \providecommand{\doi}[1]{doi: #1}\else
  \providecommand{\doi}{doi: \begingroup \urlstyle{rm}\Url}\fi

\bibitem[Aumann(1997)]{Aumann1997}
R.~J. Aumann.
\newblock {Rationality and Bounded Rationality}.
\newblock \emph{Games and Economic Behavior}, 28:\penalty0 42--67, 1997.

\bibitem[Bertsekas and Tsitsiklis(1996)]{Bertsekas1996}
D.~P. Bertsekas and J.~N. Tsitsiklis.
\newblock \emph{{Neuro-Dynamic Programming}}.
\newblock Athena Scientific, 1996.
\newblock ISBN 1886529108.

\bibitem[Braun and Ortega(2014)]{BraunOrtega2014}
D.~A. Braun and P.~A. Ortega.
\newblock {Information-theoretic bounded rationality and epsilon-optimality}.
\newblock \emph{Entropy}, 16\penalty0 (8):\penalty0 4662--4676, 2014.

\bibitem[Braun et~al.(2011)Braun, Ortega, Theodorou, and
  Schaal]{BraunOrtega2011}
D.~A. Braun, P.~A. Ortega, E.~Theodorou, and S.~Schaal.
\newblock {Path integral control and bounded rationality}.
\newblock In \emph{{IEEE Symposium on adaptive dynamic programming and
  reinforcement learning}}, pages 202--209, 2011.

\bibitem[Cesa-Bianchi and Lugosi(2006)]{CesaBianchiLugosi2006}
N.~Cesa-Bianchi and G.~Lugosi.
\newblock \emph{{Prediction, learning, and games.}}
\newblock Cambridge University Press, 2006.

\bibitem[Csisz{\'a}r and Schields(2004)]{Csiszar2004}
I.~Csisz{\'a}r and P.~C. Schields.
\newblock \emph{{Information Theory and Statistics: A Tutorial}}.
\newblock NOW Publishers, 2004.

\bibitem[Dixon(2001)]{Dixon2001}
H.~Dixon.
\newblock {Some thoughts on economic theory and artificial intelligence}.
\newblock In \emph{{Surfing Economics: Essays for the Enquiring Economist}}.
  Palgrave, 2001.

\bibitem[Duff(2002)]{Duff2002}
M.~O. Duff.
\newblock \emph{{Optimal learning: computational procedures for bayes-adaptive
  markov decision processes}}.
\newblock PhD thesis, University of Massachusetts Amherst, 2002.
\newblock Director-Andrew Barto.

\bibitem[Ellsberg(1961)]{Ellsberg1961}
D.~Ellsberg.
\newblock {Risk, Ambiguity and the Savage Axioms}.
\newblock \emph{The Quaterly Journal of Economics}, 75:\penalty0 643--669,
  1961.

\bibitem[Fox and Tishby(2012)]{Fox2012}
R.~Fox and N.~Tishby.
\newblock {Bounded Planning in Passive POMDPs}.
\newblock In \emph{{ICML}}, pages 1775--1782, 2012.

\bibitem[Friston(2009)]{Friston2009}
K.~Friston.
\newblock {The free-energy principle: a rough guide to the brain?}
\newblock \emph{Trends in Cognitive Science}, 13:\penalty0 293--301, 2009.

\bibitem[Friston(2010)]{Friston2010}
K.~Friston.
\newblock {The free-energy principle: a unified brain theory?}
\newblock \emph{Nature Review Neuroscience}, 11:\penalty0 127--138, 2010.

\bibitem[Fudenberg and Kreps(1993)]{Fudenberg1993}
D.~Fudenberg and D.~Kreps.
\newblock {Learning mixed equilibria}.
\newblock \emph{Games and Economic Behavior}, 5:\penalty0 320--367, 1993.

\bibitem[Genewein et~al.(2015)Genewein, Leibfried, Grau-Moya, and
  Braun]{Genewein2015}
T.~Genewein, F.~Leibfried, J.~Grau-Moya, and D.~A. Braun.
\newblock {Bounded rationality, abstraction and hierarchical decision-making:
  an information-theoretic optimality principle}.
\newblock \emph{Frontiers in Robotics and AI}, 2\penalty0 (27), 2015.

\bibitem[Gershman et~al.(2015)Gershman, Horvitz, and Tenenbaum]{Gershman2015}
S.~J. Gershman, E.~J. Horvitz, and J.~B. Tenenbaum.
\newblock {Computational rationality: A converging paradigm for intelligence in
  brains, minds, and machines.}
\newblock \emph{{Science}}, 349\penalty0 (6245):\penalty0 273--8, 2015.

\bibitem[Gigerenzer and Selten(2001)]{Gigerenzer2001}
G.~Gigerenzer and R.~Selten.
\newblock \emph{{Bounded rationality: the adaptive toolbox}}.
\newblock {MIT} Press, Cambridge, MA, 2001.

\bibitem[Grau-Moya et~al.(2013)Grau-Moya, Hez, Pezzulo, and
  Braun]{GrauMoya2013}
J.~Grau-Moya, E.~Hez, G.~Pezzulo, and D.~A. Braun.
\newblock {The effect of model uncertainty on cooperation in sensorimotor
  interactions}.
\newblock \emph{Journal of The Royal Society Interface}, 10\penalty0 (87),
  2013.

\bibitem[Griffiths et~al.(2015)Griffiths, Lieder, and Goodman]{Griffiths2015}
T.~L. Griffiths, F.~Lieder, and N.~D. Goodman.
\newblock {Rational use of cognitive resources: Levels of analysis between the
  computational and the algorithmic}.
\newblock \emph{Topics in Cognitive Science}, 7\penalty0 (2):\penalty0
  217--229, 2015.

\bibitem[Hansen and Sargent(2008)]{Hansen2008}
L.P. Hansen and T.J. Sargent.
\newblock \emph{{Robustness}}.
\newblock Princeton University Press, Princeton, 2008.

\bibitem[Harsanyi(1967)]{Harsanyi1967}
J.~Harsanyi.
\newblock {Games with Incomplete Information Played by "Bayesian" Players}.
\newblock \emph{Management Science}, 14\penalty0 (3):\penalty0 159--182, 1967.

\bibitem[Horvitz(1989)]{Horvitz1989}
E.~J. Horvitz.
\newblock {Reasoning about beliefs and actions under computational resource
  constraints}.
\newblock In \emph{{Uncertainty in Artificial Intelligence}}, 1989.

\bibitem[Hutter(2004)]{Hutter2004}
M.~Hutter.
\newblock \emph{{Universal Artificial Intelligence: Sequential Decisions based
  on Algorithmic Probability}}.
\newblock Springer, Berlin, 2004.

\bibitem[Kappen(2005)]{Kappen2005a}
H.~J. Kappen.
\newblock {A linear theory for control of non-linear stochastic systems}.
\newblock \emph{Physical Review Letters}, 95:\penalty0 200201, 2005.

\bibitem[Kappen et~al.(2012)Kappen, G{\'o}mez, and Opper]{Kappen2012}
H.~J. Kappen, V.~G{\'o}mez, and M.~Opper.
\newblock {Optimal control as a graphical model inference problem}.
\newblock \emph{Machine Learning}, 1:\penalty0 1--11, 2012.

\bibitem[Knight(1921)]{Knight1921}
F.H. Knight.
\newblock \emph{{Risk, Uncertainty, and Profit}}.
\newblock Houghton Mifflin, Boston, 1921.

\bibitem[Kocsis and Szepesv{\'a}ri(2006)]{Kocsis2006}
L.~Kocsis and C.~Szepesv{\'a}ri.
\newblock {Bandit based Monte-Carlo Planning}.
\newblock In \emph{{Proceedings of ECML}}, pages 282--203, 2006.

\bibitem[Legg(2008)]{Legg2008}
S.~Legg.
\newblock \emph{{Machine Super Intelligence}}.
\newblock PhD thesis, Department of Informatics, University of Lugano, June
  2008.

\bibitem[Leibfried and Braun(2015)]{Leibfried2015}
F.~Leibfried and D.~A. Braun.
\newblock {A Reward-Maximizing Spiking Neuron as a Bounded Rational Decision
  Maker}.
\newblock \emph{Neural Computation}, 27\penalty0 (8):\penalty0 1686--1720,
  2015.

\bibitem[Lewis et~al.(2014)Lewis, Howes, and Singh]{Lewis2014}
R.~L. Lewis, A.~Howes, and S.~Singh.
\newblock {Computational rationality: linking mechanism and behavior through
  bounded utility maximization}.
\newblock \emph{{Topics in Cognitive Science}}, 6\penalty0 (2):\penalty0
  279--311, 2014.

\bibitem[Leyton-Brown and Shoham(2008)]{Leyton2008}
K.~Leyton-Brown and Y.~Shoham.
\newblock \emph{{Essentials of Game Theory: A Concise Multidisciplinary
  Introduction}}.
\newblock {Synthesis Lectures on Artificial Intelligence and Machine Learning}.
  Morgan \& Claypool Publishers, 2008.

\bibitem[Lieder et~al.(2014)Lieder, Plunkett, Hamrick, Russell, Hay, and
  Griffiths]{Lieder2014}
F.~Lieder, D.~Plunkett, J.~B. Hamrick, S.~J. Russell, N.~J. Hay, and T.~L.
  Griffiths.
\newblock {Algorithm selection by rational metareasoning as a model of human
  strategy selection}.
\newblock In \emph{{Advances in Neural Information Processing Systems 27}},
  2014.

\bibitem[Lipman(1995)]{Lipman1995}
B.~Lipman.
\newblock {Information processing and bounded rationality: a survey}.
\newblock \emph{Canadian Journal of Economics}, 28:\penalty0 42--67, 1995.

\bibitem[Loomes and Sugden(1982)]{Loomes1982}
G.~Loomes and R.~Sugden.
\newblock {Regret theory: {A}n alternative approach to rational choice under
  uncertainty}.
\newblock \emph{Economic Journal}, 92:\penalty0 805--824, 1982.

\bibitem[Luce(1959)]{Luce1959}
R.~D. Luce.
\newblock \emph{{Individual choice behavior}}.
\newblock Wiley, Oxford, 1959.

\bibitem[Maccheroni et~al.(2006)Maccheroni, Marinacci, and
  Rustichini]{Rustichini2006}
F.~Maccheroni, M.~Marinacci, and A.~Rustichini.
\newblock {Ambiguity aversion, robustness, and the variational representation
  of preferences}.
\newblock \emph{Econometrica}, 74:\penalty0 1447--1498, 2006.

\bibitem[Mattsson and Weibull(2002)]{Mattsson2002}
L.-G. Mattsson and J.~W. Weibull.
\newblock {Probabilistic choice and procedurally bounded rationality.}
\newblock \emph{Games and Economic Behavior}, 41\penalty0 (1):\penalty0 61--78,
  2002.

\bibitem[McFadden(1974)]{McFadden1974}
D.~McFadden.
\newblock {Conditional logit analysis of qualitative choice behavior}.
\newblock In P.~Zarembka, editor, \emph{{Frontiers in econometrics}}. Academic
  Press, New York, 1974.

\bibitem[McKelvey and Palfrey(1995)]{McKelvey1995}
R.~D. McKelvey and T.~R. Palfrey.
\newblock {Quantal Response Equilibria for Normal Form Games}.
\newblock \emph{Games and Economic Behavior}, 10\penalty0 (1):\penalty0 6--38,
  July 1995.

\bibitem[Meginnis(1976)]{Meginnis1976}
J.~R. Meginnis.
\newblock {A new class of symmetric utility rules for gambles, subjective
  marginal probability functions, and a generalized Bayes' rule}.
\newblock In \emph{{Proceedings of the American Statistical Association,
  Business and Economic Statistics Section}}, pages 471--476, 1976.

\bibitem[Mitter and Newton(2005)]{Mitter2005}
S.~K. Mitter and N.~J. Newton.
\newblock {Information and Entropy Flow in the Kalman-Bucy Filter}.
\newblock \emph{Journal of Statistical Physics}, 118:\penalty0 145--176, 2005.

\bibitem[Mnih et~al.(2013)Mnih, Kavukcuoglu, Silver, Graves, Antonoglou,
  Wierstra, and Riedmiller]{Mnih2015}
V.~Mnih, K.~Kavukcuoglu, D.~Silver, A.~Graves, I.~Antonoglou, D.~Wierstra, and
  M.~Riedmiller.
\newblock {Human-level control through deep reinforcement learning}.
\newblock \emph{Nauture}, 518\penalty0 (518):\penalty0 529--533, 2013.

\bibitem[Mohamed and Rezende(2015)]{Mohamed2015}
S.~Mohamed and D.~J. Rezende.
\newblock {Variational Information Maximisation for Intrinsically Motivated
  Reinforcement Learning}.
\newblock In \emph{{NIPS}}, 2015.

\bibitem[Neumann and Peters(2012)]{Neumann2012}
D.~C. Neumann and J.~Peters.
\newblock {Hierarchical relative entropy policy search}.
\newblock In \emph{{Proceedings of the International Conference on Artificial
  Intelligence and Statistics}}, 2012.

\bibitem[Neumann and Peters(2013)]{Neumann2013}
D.~C. Neumann and J.~Peters.
\newblock {Autonomous reinforcement learning with hierarchical REPS}.
\newblock In \emph{{Proceedings of the International Joint Conference on Neural
  Networks}}, 2013.

\bibitem[Orseau and Ring(2012)]{Orseau2012}
L.~Orseau and Mark Ring.
\newblock {Space-Time embedded intelligence}.
\newblock \emph{Artificial General Intelligence}, pages 209--218, 2012.

\bibitem[Ortega(2011)]{Ortega2011}
P.~A. Ortega.
\newblock \emph{{A unified framework for resource-bounded autonomous agents
  interacting with unknown environments}}.
\newblock PhD thesis, Department of Engineering, University of Cambridge, UK,
  2011.

\bibitem[Ortega and Braun(2011)]{OrtegaBraun2011b}
P.~A. Ortega and D.~A. Braun.
\newblock {Information, utility and bounded rationality}.
\newblock In \emph{{Lecture notes on artificial intelligence}}, volume 6830,
  pages 269--274, 2011.

\bibitem[Ortega and Braun(2012)]{OrtegaBraun2012}
P.~A. Ortega and D.~A. Braun.
\newblock {Free Energy and the Generalized Optimality Equations for Sequential
  Decision Making}.
\newblock In \emph{{European Workshop on Reinforcement Learning (EWRL'10)}},
  2012.

\bibitem[Ortega and Braun(2013)]{OrtegaBraun2013}
P.~A. Ortega and D.~A. Braun.
\newblock {Thermodynamics as a Theory of Decision-Making with Information
  Processing Costs}.
\newblock \emph{Proceedings of the Royal Society A 20120683}, 2013.

\bibitem[Ortega and Lee(2014)]{OrtegaLee2014}
P.~A. Ortega and D.~D. Lee.
\newblock {An Adversarial Interpretation of Information-Theoretic Bounded
  Rationality}.
\newblock In \emph{{Twenty-Eight AAAI Conference on Artificial Intelligence
  (AAAI)}}, 2014.

\bibitem[Ortega et~al.(2014)Ortega, Braun, and Tishby]{Ortega2014a}
P.~A. Ortega, D.~A. Braun, and N.~Tishby.
\newblock {Monte Carlo Methods for Exact \& Efficient Solution of the
  Generalized Optimality Equations}.
\newblock In \emph{{IEEE International Conference on Robotics and Automation
  (ICRA)}}, 2014.

\bibitem[Ortega et~al.(2015)Ortega, Kim, and Lee]{OrtegaKimLee2015}
P.~A. Ortega, K.-E. Kim, and D.~D. Lee.
\newblock {Reactive Bandits with Attitude}.
\newblock In \emph{{18th International Conference on Artificial Intelligence
  and Statistics (AISTATS)}}, 2015.

\bibitem[Osborne and Rubinstein(1999)]{Osborne1999}
M.~J. Osborne and A.~Rubinstein.
\newblock \emph{{A Course in Game Theory}}.
\newblock {MIT} Press, 1999.

\bibitem[Papadimitriou and Tsitsiklis(1987)]{PapadimitriouTsitsiklis1987}
C.~H. Papadimitriou and J.~N. Tsitsiklis.
\newblock {The Complexity of Markov Decision Processes}.
\newblock \emph{Mathematics of Operations Research}, 12\penalty0 (3):\penalty0
  441--450, 1987.

\bibitem[Peters et~al.(2010)Peters, M{\"u}lling, and Alt{\"u}n]{Peters2010}
J.~Peters, K.~M{\"u}lling, and Y.~Alt{\"u}n.
\newblock {Relative entropy policy search}.
\newblock In \emph{{AAAI}}, 2010.

\bibitem[Rawlik and Toussaint(2012)]{Rawlik2012}
K.~Rawlik and M.~Vijayakumar~S. Toussaint.
\newblock {On stochastic optimal control and reinforcement learning by
  approximate inference}.
\newblock In \emph{{Proceedings Robotics: Science and Systems}}. MIT Press,
  2012.

\bibitem[Rubin et~al.(2012)Rubin, Shamir, and Tishby]{Rubin2012}
J.~Rubin, O.~Shamir, and N.~Tishby.
\newblock {Trading value and information in MDPs}.
\newblock In \emph{{Decision making with imperfect decision makers}}. Springer,
  2012.

\bibitem[Rubinstein(1998)]{Rubinstein1998}
A.~Rubinstein.
\newblock \emph{{Modeling Bounded Rationality}}.
\newblock {MIT} Press, Cambridge, MA, 1998.

\bibitem[Russell(1995)]{Russell1995a}
S.~J. Russell.
\newblock {Rationality and Intelligence}.
\newblock In Chris Mellish, editor, \emph{{Proceedings of the Fourteenth
  International Joint Conference on Artificial Intelligence}}, pages 950--957,
  San Francisco, 1995. Morgan Kaufmann.

\bibitem[Russell and Norvig(2010)]{RussellNorvig2010}
S.~J. Russell and P.~Norvig.
\newblock \emph{{Artificial Intelligence: A Modern Approach}}.
\newblock Prentice-Hall, Englewood Cliffs, NJ, 3rd edition edition, 2010.

\bibitem[Russell and Subramanian(1995)]{Russell1995b}
S.~J. Russell and D.~Subramanian.
\newblock {Provably bounded-optimal agents}.
\newblock \emph{Journal of Artificial Intelligence Research}, 3:\penalty0
  575--609, 1995.

\bibitem[Samuelson(1938)]{Samuelson1938}
P.~Samuelson.
\newblock {A Note on the Pure Theory of Consumers' Behaviour}.
\newblock \emph{Economica}, 5\penalty0 (17):\penalty0 61--71, 1938.

\bibitem[Savage(1954)]{Savage1954}
L.~J. Savage.
\newblock \emph{{The Foundations of Statistics}}.
\newblock John Wiley and Sons, New York, 1954.
\newblock ISBN 0-486-62349-1.

\bibitem[Simon(1956)]{Simon1956}
H.~A. Simon.
\newblock {Rational choice and the structure of the environment}.
\newblock \emph{Psychological Review}, 63\penalty0 (2):\penalty0 129--38, 1956.

\bibitem[Simon(1972)]{Simon1972}
H.~A. Simon.
\newblock {Theories of Bounded Rationality}.
\newblock In C.B. Radner and R.~Radner, editors, \emph{{Decision and
  Organization}}, pages 161--176. North Holland Publ., Amsterdam, 1972.

\bibitem[Simon(1984)]{Simon1984}
H.~A. Simon.
\newblock \emph{{Models of Bounded Rationality}}.
\newblock {MIT} Press, Cambridge, MA, 1984.

\bibitem[Sims(2003)]{Sims2003}
C.~A. Sims.
\newblock {Implications of rational inattention}.
\newblock \emph{Journal of Monetary Economics}, 50\penalty0 (3):\penalty0
  665--690, April 2003.

\bibitem[Sims(2005)]{Sims2005}
C.~A. Sims.
\newblock {Rational inattention: A research agenda}.
\newblock In \emph{{Proceedings of the Deutsche Bundesbank}}. Deutsche
  Bundesbank Research Center, 2005.

\bibitem[Sims(2006)]{Sims2006}
C.~A. Sims.
\newblock {Rational inattention: Beyond the linear-quadratic case}.
\newblock \emph{American Economic Review}, 96\penalty0 (2):\penalty0 158--163,
  2006.

\bibitem[Sims(2011)]{Sims2011}
C.~A. Sims.
\newblock {Rational inattention and monetary economics}.
\newblock In \emph{{Handbook of monetary economics}}. Elsevier, 2011.

\bibitem[Stengel(1994)]{Stengel1994}
R.~F. Stengel.
\newblock \emph{{Optimal Control and Estimation}}.
\newblock {Dover Books on Mathematics}. Dover Publications, 1994.

\bibitem[Still(2009)]{Still2009}
S.~Still.
\newblock {An information-theoretic approach to interactive learning}.
\newblock \emph{Europhysics Letters}, 85:\penalty0 28005, 2009.

\bibitem[Still and Precup(2012)]{Still2012}
S.~Still and D.~Precup.
\newblock {An information-theoretic approach to curiosity-driven reinforcement
  learning}.
\newblock \emph{Theory in Biosciences}, 131\penalty0 (3):\penalty0 139--148,
  2012.

\bibitem[Sutton and Barto(1998)]{Sutton1998}
R.~S. Sutton and A.~G. Barto.
\newblock \emph{{Reinforcement Learning: An Introduction}}.
\newblock MIT Press, Cambridge, MA, 1998.

\bibitem[Szepesv{\'a}ri(2010)]{Szepesvari2010}
C.~Szepesv{\'a}ri.
\newblock \emph{{Algorithms for Reinforcement Learning}}.
\newblock {Synthesis Lectures on Artificial Intelligence and Machine Learning}.
  Morgan and Claypool Publishers, 2010.

\bibitem[Tanaka et~al.(2015)Tanaka, Esfahani, and Mitter]{Tanaka2015}
T.~Tanaka, P.~M. Esfahani, and S.~K. Mitter.
\newblock {LQG Control with Minimal Information: Three-Stage Separation
  Principle and SDP-based Solution Synthesis}.
\newblock \emph{arXiv:1510.04214}, 2015.

\bibitem[Theodorou et~al.(2010)Theodorou, Buchli, and Schaal]{Theodorou2010}
E.~Theodorou, J.~Buchli, and S.~Schaal.
\newblock {A generalized path integral approach to reinforcement learning}.
\newblock \emph{Journal of Machine Learning Research}, 11:\penalty0 3137--3181,
  2010.

\bibitem[Theodorou(2015)]{Theodorou2015}
E.~A. Theodorou.
\newblock {Nonlinear Stochastic Control and Information Theoretic Dualities:
  Connections, Interdependencies and Thermodynamic Interpretations}.
\newblock \emph{Entropy}, 17\penalty0 (5):\penalty0 3352--3375, 2015.

\bibitem[Tishby and Polani(2011)]{Tishby2011}
N.~Tishby and D.~Polani.
\newblock \emph{{Perception-Action Cycle}}, chapter Information Theory of
  Decisions and Actions, pages 601--636.
\newblock Springer New York, 2011.

\bibitem[Tishby and Zaslavsky(2015)]{Zaslavsky2015}
N.~Tishby and N~Zaslavsky.
\newblock {Deep learning and the information bottleneck principle}.
\newblock In \emph{{ITW}}, pages 1--5, 2015.

\bibitem[Todorov(2006)]{Todorov2006}
E.~Todorov.
\newblock {Linearly solvable Markov decision problems}.
\newblock In \emph{{Advances in Neural Information Processing Systems}},
  volume~19, pages 1369--1376, 2006.

\bibitem[Todorov(2009)]{Todorov2009}
E.~Todorov.
\newblock {Efficient computation of optimal actions}.
\newblock \emph{Proceedings of the National Academy of Sciences U.S.A.},
  106:\penalty0 11478--11483, 2009.

\bibitem[van~den Broek et~al.(2010)van~den Broek, Wiegerinck, and
  Kappen]{Broek2010}
B.~van~den Broek, W.~Wiegerinck, and H.~J. Kappen.
\newblock {Risk Sensitive Path Integral Control}.
\newblock In \emph{{UAI}}, pages 615--622, 2010.

\bibitem[Veness et~al.(2011)Veness, Ng, Hutter, Uther, and Silver]{Veness2011}
J.~Veness, M.~Ng, M.~Hutter, W.~Uther, and D.~Silver.
\newblock {A Monte-Carlo AIXI Approximation}.
\newblock \emph{Journal of Artificial Intelligence Research}, 40:\penalty0
  95--142, 2011.

\bibitem[{Von Neumann} and Morgenstern(1944)]{Neumann1944}
J.~{Von Neumann} and O.~Morgenstern.
\newblock \emph{{Theory of Games and Economic Behavior}}.
\newblock Princeton University Press, Princeton, 1944.
\newblock ISBN 0691119937.

\bibitem[Wolpert(2004)]{Wolpert2004}
D.~H. Wolpert.
\newblock {Information theory - the bridge connecting bounded rational game
  theory and statistical physics}.
\newblock In D.~Braha and Y.~Bar-Yam, editors, \emph{{Complex Engineering
  Systems}}, chapter Information theory - the bridge connecting bounded
  rational game theory and statistical physics. Perseus Books, 2004.

\bibitem[Zilberstein(2008)]{Zilberstein2008}
S.~Zilberstein.
\newblock {Metareasoning and Bounded Rationality}.
\newblock \emph{Asociation for the Advancement of Artificial Intelligence},
  2008.

\end{thebibliography}

\end{document}